\documentclass[letterpaper]{article} % DO NOT CHANGE THIS
\usepackage[draft]{aaai25}  % DO NOT CHANGE THIS
\usepackage{times}  % DO NOT CHANGE THIS
\usepackage{helvet}  % DO NOT CHANGE THIS
\usepackage{courier}  % DO NOT CHANGE THIS
\usepackage[hyphens]{url}  % DO NOT CHANGE THIS
\usepackage{graphicx} % DO NOT CHANGE THIS
\usepackage{natbib}  % DO NOT CHANGE THIS AND DO NOT ADD ANY OPTIONS TO IT
\usepackage{caption} % DO NOT CHANGE THIS AND DO NOT ADD ANY OPTIONS TO IT
\usepackage[utf8]{inputenc} % allow utf-8 input
\usepackage[T1]{fontenc}    % use 8-bit T1 fonts

\usepackage{graphicx}
\usepackage{tikz}
\usepackage{ctable}
\usepackage{multirow}
\usepackage{subcaption}
\usepackage{xcolor}

\usepackage{amsmath}
\usepackage{amssymb}
\usepackage{mathtools}
\usepackage{amsthm}

\usepackage{stackengine}
\newcommand\barbelow[1]{\stackunder[1.2pt]{$#1$}{\rule{.8ex}{.075ex}}}

\def\R{\mathbb{R}}

\def\loss{\ell}
\def\e{\varepsilon}
\def\hypnet{\mathcal{H}}
\def\hypwei{\eta}
\def\preem{a}
\def\spm{|}

\newcommand{\norm}[2]{\left\lVert#2\right\rVert}

\DeclarePairedDelimiter\floor{\lfloor}{\rfloor}

\usepackage{algorithm}
\usepackage{algpseudocode}

\usepackage{accents}

\newcommand\munderbar[1]{%
  \underaccent{\bar}{#1}}

\newtheorem{theorem}{Theorem}[section]

\newtheorem{lemma}[theorem]{Lemma}
\newtheorem{prop}{Proposition}

\def\our{HINT}

\usepackage{newfloat}
\usepackage{listings}
\DeclareCaptionStyle{ruled}{labelfont=normalfont,labelsep=colon,strut=off} % DO NOT CHANGE THIS
\floatstyle{ruled}
\newfloat{listing}{tb}{lst}{}
\floatname{listing}{Listing}
\title{\our{}: 
Hypernetwork Approach to Training Weight Interval Regions in Continual Learning}

\title{\our{}: 
Hypernetwork Approach to Training Weight Interval Regions in Continual Learning}
\author {
    Patryk Krukowski\textsuperscript{\rm 1 \rm 2},
    Anna Bielawska\textsuperscript{\rm 2},
    Kamil Książek\textsuperscript{\rm 2}
    Paweł Wawrzyński\textsuperscript{\rm 4}
    Paweł Batorski\textsuperscript{\rm 3}
    Przemysław Spurek\textsuperscript{\rm 2}
}
\affiliations {
    \textsuperscript{\rm 1}IDEAS NCBR\\
    \textsuperscript{\rm 2}Jagiellonian University\\
    \textsuperscript{\rm 3}Heinrich Heine Universität Düsseldorf\\
    \textsuperscript{\rm 4}IDEAS Institute\\
    patryk.krukowski@ideas-ncbr.pl
}

\usepackage{bibentry}
\begin{document}

\maketitle

\begin{abstract}
Recently, a new Continual Learning (CL) paradigm was presented to control catastrophic forgetting, called  Interval Continual Learning (InterContiNet), which relies on enforcing interval constraints on the neural network parameter space. 
Unfortunately, InterContiNet training is challenging due to the high dimensionality of the weight space, making intervals difficult to manage. 
To address this issue, we introduce \our{}, a technique that employs interval arithmetic within the embedding space and utilizes a hypernetwork to map these intervals to the target network parameter space. We train interval embeddings for consecutive tasks and train a hypernetwork to transform these embeddings into weights of the target network. An embedding for a given task is trained along with the hypernetwork, preserving the response of the target network for the previous task embeddings. Interval arithmetic works with a more manageable, lower-dimensional embedding space rather than directly preparing intervals in a high-dimensional weight space. Our model allows faster and more efficient training. Furthermore, \our{} maintains the guarantee of not forgetting.      
At the end of training, we can choose one universal embedding to produce a single network dedicated to all tasks. In such a framework, hypernetwork is used only for training and, finally, we can utilize one set of weights.
\our{} obtains significantly better results than InterContiNet and gives SOTA results on several benchmarks. 
\end{abstract}


%%%%%%%%%%%%%%%%%%%%%%%%%%%%%%%%%%%%%%%%
\section{Introduction}
%%%%%%%%%%%%%%%%%%%%%%%%%%%%%%%%%%%%%%%%

Humans have a natural ability to learn from a continuous flow of data, as real-world data is usually presented sequentially. Humans need to be able to learn new tasks while also retaining and utilizing knowledge from previous tasks. While deep learning models have achieved significant success in many individual tasks, they struggle in this aspect and often perform poorly on the preceding tasks after learning new ones, which is known as catastrophic forgetting~\cite{mccloskey1989catastrophic,ratcliff1990connectionist,french1999catastrophic}.

Continual learning (CL) is an important area of machine learning that aims to bridge the gap between human and machine intelligence. While several methods have been proposed to effectively reduce forgetting when learning new tasks, such as that by~\citet{kirkpatrick2017overcoming,lopezpaz2017gradientEM,Shin2017ContinualLW,aljundi2018memory,Masse2018AlleviatingCF,Rolnick2019ExperienceRF,vandeVen2020BraininspiredRF}, they typically do not provide any solid guarantees about the extent to which the model experiences forgetting. 

In the Interval Continual Learning (InterContiNet) \cite{wolczyk2022continual}, authors propose a new paradigm that uses interval arithmetic in CL scenarios. The main idea is to use intervals to control weights dedicated to subsequent tasks. Thanks to strict interval arithmetic, these authors can enforce a network to give the same prediction for each weight sampled from the interval. Moreover, they can force the interval of the new task to be entirely contained in the hyperrectangles of the previous ones. The two above properties allow one to obtain strict constraints for forgetting. The model has strong theoretical fundamentals and obtains good results in incremental tasks and domain scenarios 
%related to regularization-based methods 
on relatively small datasets. 
% \kamil{Zostawiam komentarz, bo na razie nie wiem, jak poprawić. Nie rozumiem tego zdania. Model osiąga dobre wyniki w scenariuszach związanych z metodami regularyzacyjnymi? To są scenariusze pod metody? Poza tym pojawiające się "we" w tym akapicie sugeruje, że to dotyczy naszego modelu, a nie InterContiNetu. Może taktycznie będzie to zmienić?} \PawelW{proponuję po prostu usunąć ten wpis o regularyzacji}
The main limitation of InterContiNet is a complicated training process. To allow for the training of new tasks, one has to use large intervals in the weight space, which is trained in an extremely high-dimensional parameter space. Consequently, the model is limited to simple architectures, datasets, and continuous learning scenarios.  

\begin{figure*}
% \begin{wrapfigure}{R}{0.55\textwidth}
% \vspace{-0.5cm}
    \begin{center}
    
    \begin{tikzpicture}[scale=1.0]
    % \node[inner sep=0pt] (russell) at (-5.0,0)

    \pgftext{\includegraphics[width=0.7\textwidth]{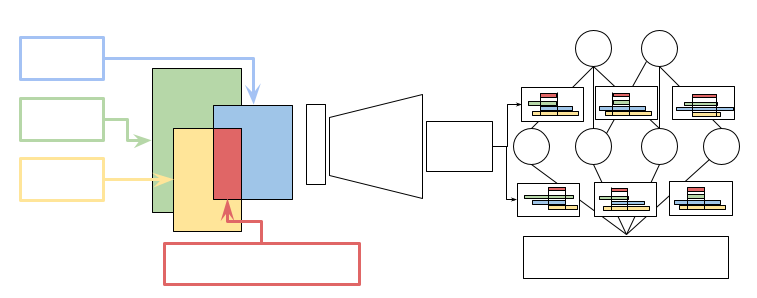}} at (0pt,0pt);
    % https://docs.google.com/drawings/d/1kO7QiqLnuYASIBVThyQ8y-dYYolnaFJnNJgYgnWot6E/edit?usp=sharing
    \node[scale=0.8] at (-0.02,0.0) {$\hypnet([\barbelow e_t, \bar e_t]; \hypwei)$};  
    \node[scale=0.8] at (-5.25,1.5) {$[\barbelow e_1, \bar e_1]$};
    \node[scale=0.8] at (-5.25,0.5) {$[\barbelow e_2, \bar e_2]$};
    \node[scale=0.8] at (-5.25,-0.5) {$[\barbelow e_3, \bar e_3]$};
    \node[scale=0.8] at (-1.9,-1.9) {$[\barbelow{e}, \bar{e}] = \bigcap_{t=1}^T [\barbelow{e}_t, \bar{e}_t]$};
    \node[scale=0.8] at (1.4,0.0) {$[\barbelow \theta_t, \bar \theta_t]$};
    \node[scale=0.8] at (4.1,-1.8) {$[\barbelow z_{x,t}, \bar z_{x,t}] = \phi(x; [\barbelow \theta_t, \bar \theta_t])$};

    % \node[scale=0.8, rotate=90] at (-9.2,0.0) {Semi-binary Mask };
    % \node[scale=0.9] at (-7.0,0.0) {$\bigotimes$ };
    % \node[scale=1.0] at (2.1,0.0) {\bf $=$ };
    % \node[text width=7cm] at (4.,-5.5) { \color{OliveGreen} Non-Trainable parameters (input images) };
    \end{tikzpicture}
    \end{center}
    % \PawelW{Coś w tym rysunku trzeba zmienić, bo właściwie nie widać, dlaczego to przecięcie jest akurat tu, gdzie jest. czerwony jest częścią wspólną 1 i 3, a 2 jest dopasowane do nich - nie wiadomo czemu.}
    % \KamilMain{Minimum, co można zrobić, to dodać przezroczystość do tych prostokątów. Kiedyś proponowałem drugą część tego rysunku, gdzie miało być przecięcie embeddingów na kolejnych współrzędnych. Oprócz tego TeXowe wzory są tutaj niewyraźne}
    % \vspace{-0.5cm}
    \caption{\our{} uses interval arithmetic in the input to the hypernetwork. After propagating the intervals through the hypernetwork, we obtain the intervals on the target network layers. The intersection of all intervals produces universal embeddings dedicated to all tasks. Our model gives theoretical guarantees for not forgetting.
    % In our method, we combine the Interval Bound Propagation \cite{gowal2018effectiveness} paradigm in the hypernetwork and InterContiNet \cite{wolczyk2022continual} in the target network.
    }
    % \vspace{-0.5cm}
    \label{fig:merf_appro}
% \end{wrapfigure}
\end{figure*}
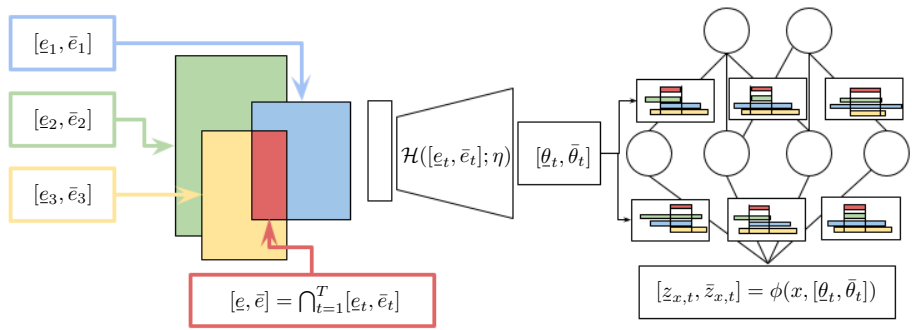

To solve such a problem, we propose \our{}\footnote{The source code is available at https://github.com/gmum/HINT}, i.e. a model which uses intervals in the embedding space and a hypernetwork to transfer them to the weight space of the target network, see Fig.~\ref{fig:merf_appro}.
Such a solution allows one to effectively train 
interval-based target network on large datasets and the most difficult incremental class scenarios, surpassing its foregoing limitations.
A hypernetwork architecture~\cite{ha2016hypernetworks} is defined as a neural network that generates weights for a specific target network designed to solve a particular problem. In continual learning, a hypernetwork may generate weights for the target model~\cite{von2019continual,henning2021posterior,ksikazek2023hypermask}. Dependent on the task identity, trainable embeddings are fed into the hypernetwork. After training is complete, a single meta-model can create task-specific weights. This ability to generate distinct weights for various tasks enables models based on hypernetworks to maintain minimal knowledge loss.

In \our{}, we use interval arithmetic in an embedding space. 
The hypernetwork is fed with task specific interval embedding to produce weights for the target network. Our architecture uses a target model with interval weights and Interval Bound Propagation (IBP) \cite{gowal2018effectiveness} technique in the hypernetwork. Due to low dimension of the task embeddings, this solution provides greater robustness to forgetting than previous approaches to CL based on interval arithmetic. 

%PW: było: Our \our{} produce intervals in a low-dimensional embedding space instead of a high-dimensional parameter space \kamil{Może warto jeszcze raz w tym miejscu w zdaniu lub dwóch dodać, dlaczego się na to decydujemy?}.  

Our model can be used in incremental class learning in both scenarios: with and without task identity. We can use the evaluation scenario proposed by hypernetwork-based models~\cite{von2019continual,henning2021posterior,ksikazek2023hypermask}. 
For known task identity scenarios, we use interval arithmetic as a regularization technique, and each task embedding produces new target network interval weights dedicated to this task. 
We can also use an entropy criterion to determine a class identity when it is unknown. 
In both solutions, we must memorize hypernetwork and all task embeddings. 
Our framework allows us to produce a universal embedding.
With interval arithmetic, we can control intervals and transform them into interval embeddings with non-empty intersections. % The intersection of intervals, i.e. a universal embedding, can be used for solving multiple tasks. In such a scenario, we use the hypernetwork only as a {\em meta-trainer}, which is not used in inference.
At the end of the training, we generate the final target network from the universal embedding created from the intersection of all previously trained interval embeddings, being able to solve all tasks. Such an approach significantly outperforms InterContiNet and in many cases outperforms state-of-the-art methods. 
% \PatrykSide{Tutaj trzeba podkreślić, że niepuste przecięcia nie powstają tylko i wyłącznie dzięki arytmetyce interwałowej, ale głównie przez to, że pompujemy interwały w przestrzeni, która powstaje przez zmapowanie punktów z interwału przez odwzorowanie $f\left(x\right) = \cos\left(x\right)$}

Our contributions can be summarized as follows:
\begin{itemize}
    \item We show \our{}, which uses interval arithmetic in the embedding space and a hypernetwork to propagate intervals to the target network weight space.
    \item We demonstrate that our model can efficiently use interval arithmetic in CL settings with large datasets.
    \item We demonstrate that a hypernetwork can be used to recursively refine, across consecutive tasks, the region of target network weights that are universally effective for all tasks.
\end{itemize}

%%%%%%%%%%%%%%%%%%%%%%%%%%%%%%%%%%%%%%%%%%%%%%%%%%%%%%%%%%
\section{Related Works}
%%%%%%%%%%%%%%%%%%%%%%%%%%%%%%%%%%%%%%%%%%%%%%%%%%%%%%%%%%

\paragraph{Continual learning.}
When a neural network is trained on new data, it often forgets patterns learned from previous data, a problem known as \emph{catastrophic forgetting} \cite{french1999catastrophic}. Continual learning methods enable the network to learn new tasks without losing performance on earlier ones, even when previous data is no longer available. Different setups exist for continual learning~\cite{2019vandeven+1}: in task-incremental learning, the task is known during inference; in class-incremental learning, the network must learn a consistent input-output distribution across tasks with varying data distributions.

% Suppose we have trained a~neural network with a~chunk of data and then we apply a typical training routine to the same network with another chunk of data. That usually leads to {\it catastrophic forgetting} \cite{french1999catastrophic}, i.e., the network ceases to represent patterns trained on the earlier chunk of data. {\it Continual learning} is a set of training methods that allow the network to adjust to consecutive chunks of data, also referred to as {\it tasks}, while retaining fitness to previous chunks, now inaccessible. Several setups for continual learning are considered \cite{2019vandeven+1}. In a task-incremental setup, we assume that in the inference a data sample belongs to a certain known task. In a class-incremental setup, we assume that the tasks may vary in data distribution, yet share a common input-output conditional distribution that the network must learn to represent.

Rehearsal methods store a limited-size representation of previous tasks and train the network on a combination of data samples from the previous and current tasks. In the simplest setup, a memory buffer is used to store selected previous samples in their raw form \cite{lopezpaz2017gradientEM,2019aljundi+6,2020prabhu+2} or store exemplars of classes in the class incremental (CIL) setting \cite{2017rebuffi+3,2018chaundry+3,2018castro+4,2019wu+5,2019hou+5,2019belouadah+1}. Some methods instead of using raw data samples, store data representations of previous tasks in different
forms, e.g., optimized artificial samples \cite{2020liu+4}, distilled datasets \cite{2018wang+3,2021zhao+2} or addressable memory structures \cite{2022deng+1}. 

Buffer-based approaches raise questions of scalability and data privacy. To address these questions, generative rehearsal methods are based on generative models that produce data samples similar to those present in previous tasks, instead of just replaying them from a~storage. In Deep Generative Replay \cite{Shin2017ContinualLW}, a Generative Adversarial Network (GAN) is used as a model to generate data from the previous tasks. For the same purpose, Variational Autoencoders (VAE) are used in \cite{2018vandeven+1}, Normalizing Flows in \cite{2018nguyen+3}, Gaussian Mixture Models in \cite{2019rostami+2} and Diffusion Models in \cite{2024cywinski+4}. 

Regularization methods \cite{kirkpatrick2017overcoming,2017zenke+2,li2017learning} are designed especially for the task incremental (TIL) setup. They identify neural weights with the greatest impact on task performance and slow down their learning on further tasks through regularization.

Architectural methods \cite{rusu2016progressive,2018yoon+3,mallya2018packnet,Mallya2018PiggybackAA,2020wortsman+6} are especially designed for the TIL setup. They adjust the structure of the model for each task.

\citet{wolczyk2022continual} introduced the InterContiNet algorithm. For each task, this method identifies a region in the weight space in which any weight vector assures a good performance. A weight vector proper for all tasks is taken from the intersection of these regions. In InterContiNet, these regions take the form of hyperrectangles. The training of InterContiNet in a continual learning setting is problematic since intervals in the high-dimensional parameter space need to be controlled. To alleviate these problems, the authors engineered an elaborate training procedure. However, in a multidimensional weight space, it is difficult to optimize weight hyperrectangles for different tasks to have a non-empty intersection. Consequently, InterContinNet is limited to relatively simple architectures, datasets, and CL scenarios. 
%The above optimization procedure is exposed to many problems in practical applications. In a multidimensional weight space, it is difficult to optimize weight hyperrectangles for different tasks to have a non-empty intersection. Consequently, the model is limited to relatively simple architectures, datasets, and CL scenarios. To solve the above problem, we propose \our{}, a model which uses intervals in the embedding space and a hypernetwork to produce intervals in the weight space, as presented in Fig.~\ref{fig:merf_appro}. In \our{}, the weights of the target networks are produced by a meta-model, which is trained for all tasks and generates dedicated weights for each subsequent task. \our{} allows one to effectively train InterContiNet on large datasets and the most difficult incremental class scenarios. In addition, we can retain all InterContiNet properties and maintain guarantees of not forgetting. 
In \cite{henning2021posterior}, the regions in the weight space take the form of normal distributions. In the architecture presented there, a hyper-hypernetwork transforms inputs into task-specific embeddings which a hypernetwork transforms into those distributions in the weight space of the target network. The current work develops the same emerging family of methods that could be called regional methods. However, here we consider regions that result from the transformation of low-dimensional cuboids assigned to tasks by a~hypernetwork. This way, we overcome the issue of scalability and present an architecture capable of dealing with complex architectures, datasets and CL scenarios.

% \paragraph{Interval arithmetic for neural networks}
% In deep learning, interval arithmetic was used in three main areas. First of all, we can use interval arithmetic to deal with the situation where we have only uncertain information about the data.  
% \citet{chakraverty2014interval} presented an interval artificial neural network (IANN) that can handle input and output data represented as intervals. This architecture was further adapted for viruses data with uncertainty modeled by intervals~\cite{chakraverty2017novel,sahoo2020structural}.

% In~\cite{gowal2018effectiveness,morawiecki2019fast,mirman2019provable}, interval arithmetic was used to produce neural networks that were robust against adversarial attacks. In such an approach, authors used worst-case cross-entropy to train the classification model.
% In~\cite{proszewska2021hypercube}, the architecture was employed for representing voxels in 3D shape modelling.

% In InterContiNet \cite{wolczyk2022continual}, the interval arithmetic was applied in a neural network with interval weights. This network transforms a given input into a loss interval. In training, the weight intervals are optimized to minimize the average loss for the weight interval center, while also limiting the average loss upper bound on the weight intervals. 

% In contrast to these previous works, we apply interval arithmetic to weights instead of input data. 

%%%%%%%%%%%%%%%%%%%%%%%%%%%%%%%%%%%%%%%%
%\section{Interval arithmetic in Continual Learning scenario}
%%%%%%%%%%%%%%%%%%%%%%%%%%%%%%%%%%%%%%%%

% ---------------------------------------------------\\

\section{Method: \our{}}

Our proposed CL architecture, dubbed \our{}, is presented in Fig.~\ref{fig:merf_appro}, and based on the following logic. Hyperrectangular, learnable task embeddings are fed to the hypernetwork, which produces weights for the target network, which solves the CL task. Both networks use the interval arithmetic. The intersection of all task interval embeddings forms a universal interval embedding, from which every embedding is suitable for all CL tasks. The training procedure preserves the performance of the architecture on previous tasks. Below, we present separate parts of this architecture. We commence with a presentation of the target network and the principles of interval neural networks, and then we describe the interval hypernetwork. 

\subsection{Interval target network}

%We commence with describing InterContiNet \cite{wolczyk2022continual} (Interval Continuous Learning), which shows how interval arithmetic can be used in CL scenarios. 

In interval neural networks, (\citeauthor{dahlquist2008numerical} \citeyear{dahlquist2008numerical}, Sec. 2.5.3, \citeauthor{wolczyk2022continual} \citeyear{wolczyk2022continual})
instead of considering particular points $\vartheta \in \mathbb{R}^D$ in the parameter space, regions $\Theta \subset \mathbb{R}^D$ are used. 
%In InterContiNet, the authors use interval arithmetic \cite{dahlquist2008numerical} (Section 2.5.3) and operations on segments (hyperrectangles).
The hyperrectangle $[\barbelow \theta, \bar \theta]$ is a Cartesian product of one-dimensional intervals
$$
[\barbelow \theta, \bar \theta] = [\barbelow \theta^{(1)}, \bar \theta^{(1)}] \times [\barbelow \theta^{(2)}, \bar \theta^{(2)}] \times \ldots \times [\barbelow \theta^{(D)}, \bar \theta^{(D)}] \subset \mathbb{R}^D, 
$$ 
where by $\theta^{(i)} \in [\barbelow \theta^{(i)}, \bar \theta^{(i)}]$ we denote the $i$-th element of $\theta$. 
Interval arithmetic utilizes operations on segments. 
By $\phi(x;\theta)$ we denote the target network, which is a layered neural classifier with input $x$ and weights $\theta$. 
%In InterContiNet \cite{wolczyk2022continual}, weights of the network are assumed 
We assume the weights $\theta$ to be intervals, thus altogether define a hyperrectangle, $[\barbelow \theta, \bar \theta]$. Therefore, for a given input $x$, the network produces a hyperrectangular output 
$
    [\barbelow z, \bar z] = \phi(x; [\barbelow \theta, \bar \theta]), 
$
rather than a vector. Appendix \ref{sec:ia-in-nns} provides details of the interval network inner workings. 

%PawełW: zakomentowałem, bo nie wiem, co nowego to wnosiło do opisu
%Using this framework, interval arithmetic can be applied to propagate any input $z_0$ across the entire network, resulting in $N$ interval logits for $N$ different classes $[\underline{z}_L, \overline{z}_L]$. 
%Due to precise computations using interval arithmetic, for $(W_1, \ldots, W_L) \in [\barbelow W_1, \bar W_1]\times \ldots\times [\barbelow W_L, \bar W_L ]$
%we have
%$$
%\phi(x; W_1, \ldots, W_L) \in [\underline{z}_L, \overline{z}_L] = \phi(z_0;[\barbelow W_L, \bar W_L],\ldots, [\barbelow W_1, \bar W_1 ]).
%$$

Since intervals are used instead of points, the worst-case cross-entropy is used instead of classical ones. 
The worst-case interval-based loss is defined by
\begin{equation}
\hat{\loss}(x, y; \left[\barbelow\theta, \bar\theta\right]) =  \loss(\hat{z}, y), 
\end{equation}
where $x$ is an observation, $y$ is an~one-hot encoded class of $x$, $\loss\left(\cdot, \cdot\right)$ is the cross-entropy loss, and $\hat{z}$ is a vector with the $i$-th element defined as:
$$
\hat{z}^{(i)} = \left\{
\begin{array}{ll}
     \bar{z}^{(i)}, & \mbox{ for } y^{(i)}=0,\\[0.8ex]
     \barbelow{z}^{(i)}, & \mbox{ for  } y^{(i)}=1,
\end{array}
\right.
$$
where $z = \phi(x; \left[ \barbelow\theta, \bar\theta \right])$. 
The worst-case cross-entropy, as shown in \cite[Theorem 3.1]{wolczyk2022continual}, gives a strict upper limit on cross-entropy 
\begin{equation}
\hat{\loss}(x, y; [\barbelow \theta, \bar \theta]) \geq \max_{\theta \in [ \munderbar \theta, \bar \theta]} \loss(\phi(x;[\theta,\theta]), y).
\end{equation}

%The  InterContiNet network is applicable in continual learning settings, ensuring no forgetting. Practically, 
For the continual learning challenge of optimizing across the tasks $1, \ldots, T$, during the training on a specific task $t$, the goal is to achieve the optimal worst-case cross-entropy within the intervals $[\barbelow \theta_t, \bar \theta_t]$ where $[\barbelow \theta_{t}, \bar \theta_{t}]\subseteq[\barbelow \theta_{t-1}, \bar \theta_{t-1}]$. Interval arithmetic facilitates non-forgetting by maintaining the weights interval of task $t$ within the weights interval trained for task $t-1$.

%Concluding,  InterContiNet utilizes a conventional feed-forward neural network $\phi$ with parameters $\theta$, but rather than focusing on specific points in the parameter space, it considers regions $[\barbelow\theta,\bar\theta]$. Specifically,  InterContiNet 
Technically, the interval classifier $\phi$ employs for task $t$ the parameter of the network with the central point $\theta_t$ and an interval radius $\e_t\in\mathbb{R}^D$. Therefore, the parameter region is defined as 
$[\barbelow \theta_t, \bar \theta_t] = [\theta_t - \e_t, \theta_t + \e_t].$ 
With this approach, the network can still function as a conventional non-interval model by using only the central weights $\theta_t$, which consequently generates only the central activations for each layer. Concurrently, the activation intervals $[\barbelow z_l, \bar z_l]$ can be computed with interval arithmetic. As detailed by \citet{gowal2018effectiveness}, these processes can be efficiently implemented on GPUs by simultaneously computing the central activations and their corresponding radii.  %InterContiNet 
The interval neural network redesigns the fundamental components of neural networks (such as fully-connected or convolutional layers, activations, and pooling) to accommodate interval inputs and weights.

\subsection{Interval hypernetwork}\label{sec:interval_hypernetwork}
%%%%%%%%%%%%%%%%%%%%%%%%%%%%%%%%%%%%%%%%

In this section, we introduce a hypernetwork that transforms hyperrectangles in a~low-dimensional embedding space into regions in the weight space of the target network. We demonstrate the effectiveness of this joint architecture. Finally, we show that the hypernetwork may be used only in training, while during inference only the target network is utilized. A region in the target network weight space is shaped for all CL tasks iteratively, i.e. new constraints are defined for subsequent tasks. Finally, we can create a universal region of weights with a universal embedding vector. Still, we can produce weights for each task with dedicated embeddings, ensuring higher performance for individual tasks, but at the expense of less universality.

\our{} consists of the interval hypernetwork directly preparing weights for the interval target network which finally performs a classification. The consecutive parts of the \our{} architecture will be described in the next few paragraphs.

% \KamilSide{Trochę robiliśmy to już w poprzedniej sekcji} We start from the hypernetwork paradigm in the CL setting. Then we show that such a framework can be integrated with InterContiNet. Finally, the hypernetwork can be seen as a meta-trainer used only in training. 

\paragraph{Hypernetwork}

Introduced in \cite{ha2016hypernetworks}, hypernetworks are neural models that produce weights for a distinct target network designed for solving a specific problem. %The objective is to minimize the number of trainable parameters by creating a hypernetwork possessing fewer parameters than the target network. 
Hypernetworks were already used in continual learning~\cite{von2019continual,henning2021posterior,ksikazek2023hypermask}.
In this context, they generate unique weights for individual CL tasks.

HNET~\cite{von2019continual} introduces trainable embeddings ${e}_t \in \mathbb{R}^{M}$, one for each task $t$, where $t \in \lbrace 1, ..., T \rbrace$. The hypernetwork, denoted as $\hypnet$, and parameterized by $\hypwei$, outputs the weights specific to the $t$-th task, $\theta_t$, for the target network $\phi$, as shown in the equation: 
$
\hypnet(e_t; \hypwei) = \theta_t.
$

% The HNET's meta-architecture, or hypernetwork, is designed to generate unique weights for each task in continual learning.
The function $\phi(\cdot\,;\theta_t) : X \to Y$
%\PatrykSide{Moim zdaniem, jeżeli używamy $\theta_t$ w definicji tego odwzorowania, to powinniśmy także pisać wszędzie $X_t$ oraz $Y_t$, by zaznaczyć, że para $\left(X_t, Y_t\right)$ pochodzi z $t$-tego tasku. PW: no nie, bo $X$ i $Y$ to tu są domenty we/wy i one są takie same dla każdego tasku}
represents a neural network classifier with weights $\theta_t$ generated by the hypernetwork $\mathcal{H}(\cdot\,;\hypwei)$ with weights $\eta$, and assigns labels for samples in a given task. Notably, the target network itself is not trained directly. Within HNET, a hypernetwork $\hypnet(\cdot\,;\hypwei):~\mathbb{R}^M \ni {e}_t  \mapsto \theta_t$ computes the weights $\theta_t$ for the target network $\phi$ based on the dedicated task embedding ${e}_t$. Consequently, each task in continual learning is represented by a classifier function
$
%\phi_{\theta_t} = 
\phi(\cdot;\theta_t) = \phi\big( \cdot ; \hypnet({e}_t ; \hypwei)\big). 
$ 

After training, a single meta-model is produced, which furnishes weights for specific tasks. The capacity to generate distinct weights for each task allows hypernetwork-based models to exhibit minimal catastrophic forgetting. When learning the following task, essentially a new architecture is created, with weights dedicated to this task. To identify weights fitting to all tasks, we introduce \our{}, which uses a hypernetwork with interval arithmetic and, optionally, also training rules ensuring a joint weight subregion for more CL tasks. Basically, this scenario corresponds to a single architecture for a higher number of tasks. After training, we propagate the intersection of intervals through the hypernetwork to produce universal weights for all tasks. In such a scenario, we do not need to memorize embeddings and hypernetworks since one network is dedicated to all tasks.  
% \PawelWSide{Prośba o rozwinięcie tego ostatniego zdania, bo nie jest jasne, jaki to scenariusz i o jaką jedną architekturę chodzi.}

\paragraph{Interval arithmetic in the embedding space.}

In HNET~\cite{von2019continual} and HNET-based models \cite{henning2021posterior,ksikazek2023hypermask} authors use one-dimensional trainable embeddings ${e}_t \in \mathbb{R}^{M}$ for each task $t$, where $t \in \lbrace 1, ..., T \rbrace$. In \our{}, we use interval arithmetic, which results in an embedding defined by its lower and upper bound for each coordinate:
% \PatrykMain{Tak, brakuje plusów. Ponadto trzeba zdefiniować, czym jest $\epsilon_t$ i oczywiście napisać, że taki wektor musi mieć nieujemne współrzędne. Zamiast $e_t^i$ pisałbym $e_t^{\left(i\right)}$, żeby ktoś nie pomyślał, że podnosimy współrzędne do kolejnych potęg :)}
$
[\barbelow e_t, \bar e_t] = [e_t - \e_{t,i}, e_t + \e_{t,i}] =  [e^{(1)}_t - \e^{(1)}_{t, i}, e^{(1)}_t + \e^{(1)}_{t, i}] \times \ldots \times  [e^{(M)}_t - \e^{(M)}_{t, i}, e^{(M)}_t + \e^{(M)}_{t, i}] \subset \mathbb{R}^M,
$
where $\e_{t} = [\e_{t, i}^{(1)}, ..., \e_{t, i}^{(M)}]$ is a perturbation vector during the $i$--th iteration of training the $t$--th task ($i \in \lbrace 1, ..., n \rbrace$) satisfying the condition
$
\sigma\left(\epsilon_{t}\right) = 1,
$ 
where $\sigma\left(\cdot\right)$ is the softmax function. It is worth noting that such a normalization technique is applied to perturbated vectors before passing through the hypernetwork. Consequently, $e_t$ denotes the center of the embedding for the $t$--th task. The presented condition ensures that intervals do not collapse to zero widths in training. Perturbation vector coordinates values are trainable when we create specific weights for each task or are strictly given in cases where we define a joint weight subregion for all tasks. Nevertheless, its values must be non-negative, i.e., $\e_{t, i}^{j} \geq 0$ for $j \in \lbrace 1, ..., M \rbrace$.

Thanks to using interval-based embeddings, we can select an embedding subspace to create regions dedicated to more CL tasks, see Fig.~\ref{fig:merf_appro}. In \our{}, interval embeddings are transformed by the hypernetwork into hyperrectangles of weights of the target model. In this scenario, the hypernetwork propagates segments instead of points. 
To achieve this, in the hypernetwork model, we use an architecture based on Interval Bound Propagation (IBP)~\cite{gowal2018effectiveness}.

% IBP \cite{gowal2018effectiveness} was proposed to construct provably robust classifiers. \KamilSide{Ten akapit musi przejść do sekcji powyżej, ale na razie go nie przenoszę, bo trzeba znaleźć mu miejsce.} It employs interval arithmetic to transmit an axis-aligned bounding box across layers, aiming to reduce the upper bound of the difference between any two logits when the input undergoes perturbation within a bounded-normed sphere.\PawelWSide{No faktycznie, nie bardzo to tu pasuje}

In our proposed \our{}, we use a hypernetwork whose weights $\hypwei$ are vectors, but we propagate an interval input $[\barbelow e_t, \bar e_t]$ producing an interval output, i.e., 
$$
    [\barbelow \theta_t, \bar \theta_t] = \hypnet([\barbelow e_t, \bar e_t]; \hypwei). 
$$
The inner workings of the hypernetwork are based on interval arithmetic. We elaborate more on this in Appendix \ref{sec:hypernet-with-em-in}. 

The hypernetwork $\hypnet$ is trained using the procedure described below. It is necessary to ensure that in learning a given task, \our{} does not forget the previously learned knowledge. 
Outputs from task-conditioned hypernetworks are generated based on the task embedding.
To prevent the hypernetwork from forgetting previous tasks, we regularize its training to produce the same target network weights for previous task embeddings. In a training of task $T$ with \our{}, the regularization loss is specified as:
%We can compare the hypernetwork output, generated before the learning of task $T$ with weights $[\munderbar{\theta^{*}}, \bar{\theta^{*}}]$, against the output following the proposed changes in the hypernetwork weights $\Delta \eta$. Lastly, in \our{}, the output regularization loss is specified as:
% \begin{align*}
% \mathcal{L}_{output} (\hypwei^*, \hypwei, \Delta \hypwei, \{ [\barbelow e_t, \bar e_t]  \}_{t=1\dots T-1}  )
% =  \frac{1}{3\left(T-1\right)} & \sum_{t=1}^{T-1} \left( \left\|  \hypnet(\barbelow e_t; \hypwei^{*}) - \hypnet(\barbelow e_t; \hypwei + \Delta \hypwei) \right\|^2 \right. \\
% \quad + \left\| \hypnet\left(e_t; \hypwei^{*}\right) - \hypnet\left( e_t; \hypwei + \Delta \hypwei \right) \right\|^2 &
% \quad + \left. \left\|  \hypnet(\bar e_t; \hypwei^{*}) - \hypnet(\bar e_t; \hypwei + \Delta \hypwei) \right\|^2 \right) .
% \end{align*}
%$$
%\mathcal{L}_{output}(\eta)
%=  \frac{1}{3\left(T-1\right)} \sum_{t=1}^{T-1} \left( \left\|  \hypnet(\barbelow e_t; \hypwei_{T-1}) - \hypnet(\barbelow e_t; \hypwei) \right\|^2 \right. 
%+ \left\| \hypnet\left(e_t; \hypwei_{T-1}\right) - \hypnet\left( e_t; \hypwei\right) \right\|^2 
%+ \left. \left\|  \hypnet(\bar e_t; \hypwei_{T-1}) - \hypnet(\bar e_t; \hypwei) \right\|^2 \right) .
%$$
%\begin{equation}
$$
\mathcal{L}_{output}(\eta)
=  \frac{1}{3\left(T-1\right)}  \sum_{t=1}^{T-1} 
\!\!\!\!\!\!\!\!\!\!\!\!\!\!\!\!\!\!\!\!\!\!\!\!\!\!
\sum_{\qquad\qquad\mu \in \{\munderbar e_t, \frac{\munderbar e_t + \bar e_t}{2}, \bar e_t\}} 
\!\!\!\!\!\!\!\!\!\!\!\!\!\!\!\!\!\!\!\!\!\!\!\!\!\!
\left\|  \hypnet(\mu; \hypwei_{T-1}) - \hypnet(\mu; \hypwei) \right\|^2,  
$$
%\end{equation}
where $\hypwei_{T-1}$ are the hypernetwork weights trained for task $T-1$. 
The second component of the nested sum above corresponds to the regularization of the embedding center, i.e. $e_t = (\munderbar{e}_t + \bar{e}_t)/2$. The motivation for using this regularization formula is its ability to effectively preserve the knowledge acquired from previous tasks by controlling the interval lengths produced by the hypernetwork. Specifically, when the product of the hypernetwork's weights is small, the resulting intervals tend to be short. This outcome is a direct consequence of the Lipschitz continuity of MLP networks, which we use exclusively as hypernetworks. This made additional regularization redundant. For further validation, please refer to Appendix \ref{appendix:interval_lengths_plots}.

%Although small weights can typically be achieved through $L^2$ or $L^1$ norm regularization, we found that in our case, this was unnecessary. The intervals generated by the hypernetwork were sufficiently small even without these regularizations. For further validation, please refer to Appendix \ref{appendix:interval_lengths_plots}.

% \PawelW{To zajmuje tutaj dużo miejsca, a jest to generalny wynik oznaczający, że funkcja reprezentowana przez MLP jest Lipschitz-ciągła ze stałą proporcjonalną do iloczynu norm wag warstw. Poza tym, jest tu powołanie się na (4), które jest dopiero później. Ja bym to zrobił inaczej: przerzucił Proposition do appendixa D, przemianował go na coś w rodzaju "Lipschitz-continuity of MLP", a tu napisał, że korzystamy tutaj z Lipshitz-ciągłości hypernetworka w taki to a taki sposób, a dowodzimy jej w appendiksie.  }
% \PawelW{To chyba też łatwo da się uogólnić do dowolnych funkcji aktywacji, wstawiając w iloczyn maksymalną pochodną funkcji aktywacji. }
The proposition and the proof of the Lipschitz continuity of MLP networks is detailed in Appendix \ref{sec:lipschitz_proof}. Consequently, regularizing the endpoints of the intervals alone is sufficient to maintain the knowledge from previous tasks. This conclusion is further supported by the fact that the proposed regularization imposes a non-increasing constraint on the interval. As a result, it is unnecessary to consider other points within the interval $\left[\barbelow{e}_t, \bar{e}_t\right]$ for regularization purposes. Nonetheless, we also apply additional regularization to the center of the interval, as this approach has been observed to yield slightly better results.
When the task identity is given in the inference stage, then one can just use 
% $$
% \mathcal{L}_{output} (\hypwei^*, \hypwei, \Delta \hypwei, \{ [\barbelow e_t, \bar e_t]  \}_{t=1\dots T-1}  )
% =  \frac{1}{T-1} \sum_{t=1}^{T-1}\left\| \hypnet\left(e_t; \hypwei^{*}\right) - \hypnet\left( e_t; \hypwei + \Delta \hypwei \right) \right\|^2 ,
% $$
$$
\mathcal{L}_{output} (\hypwei)
=  \frac{1}{T-1} \sum_{t=1}^{T-1}\left\| \hypnet\left(e_t; \hypwei_{T-1}\right) - \hypnet\left( e_t; \hypwei\right) \right\|^2 ,
$$
as we do not have to regularize the hypernetwork beyond the middle of the interval. The ultimate cost function is a sum of a component $\mathcal{L}_{current}$, defined by the current task data, and output regularization $\mathcal{L}_{output}$, i.e.,  
\begin{equation} \label{main:objective}
\mathcal{L} = \mathcal{L}_{current} + \beta \cdot \mathcal{L}_{output}, 
\end{equation}
where $\beta$ represents a hyperparameter managing the intensity of regularization. For an input-output data pair, $(x,y)$, the current loss is defined as the standard cross-entropy combined with the worst-case cross-entropy, i.e., 
$$
\loss_{current} = \kappa\cdot \loss\left(\frac{\barbelow z_L + \bar z_L }{2}, y \right) + \left(1-\kappa\right) \cdot \loss\left( \hat z_L, y \right),
$$
where $\kappa$ is a hyperparameter scheduled during training helping to control the proper classification of samples. $\mathcal{L}_{current}$ in \eqref{main:objective} is the average of $\loss_{current}$ over the current task data $D_t$. 

Both components of $\mathcal{L}$ are essential because \our{} includes two networks, and it is imperative to mitigate drastic changes in the hypernetwork output weights after learning of subsequent CL tasks while being able to gain new knowledge.
The pseudocode of \our~is presented in Appendix~\ref{sec:algorithm}.

\paragraph{Hypernetwork as a meta-trainer}

\our{} consists of interval embeddings propagated through the IBP-based hypernetwork and the target network. % Thus, the hypernetwork can be seen as a meta-trainer. %During training, we simultaneously optimize embedding and hypernetwork.
During the training of CL tasks, we ensure that consecutive embedding intervals have a non-empty intersections with the previous ones. A common part of these embeddings may be used as a universal embedding and applied for solving all CL tasks, see Fig.~\ref{fig:interval_embeddings_cifar100}. Therefore, we can propagate it through the IBP-based hypernetwork, getting a single target network to classify samples from all tasks. In such a way, the hypernetwork does not have to be used in inference and thus is considered as a \emph{meta-trainer}. Finally, one set of weights can be utilized without storing the previous ones.

In order to achieve embedding hyperrectangles overlapping for different tasks, instead of training them directly, we generated them from trained {\it pre-embeddings}, $\preem_t \in \mathbb{R}^M$, with the following formulae 
\begin{equation} \label{preem2em}
    \begin{split}
    e_t & = (\gamma/M) \cos(\preem_t), \\ 
    \barbelow e_t & = e_t - \gamma\cdot\sigma(\epsilon_t),  \\ 
    \bar e_t & = e_t + \gamma\cdot\sigma(\epsilon_t),     
    \end{split}
\end{equation}
where $\gamma$ is a perturbation hyperparameter and $M$ is a natural number representing the embedding space dimensionality, and $\sigma(\cdot)$ is the softmax function. %\kamil{Mam wrażenie, że ten akapit jest trochę tajemniczy. Nagle pojawiają się pre-embeddingi i nie do końca wiadomo, czemu one mają służyć. Może warto, żebyśmy opisali krótko ich sens?} \PawelW{Czy teraz jest lepiej?}
%We achieve nested embeddings by transforming embedding centers $e_t$ using $\frac{\gamma}{M} \cdot \cos(e_t)$ derived in Lemma~\ref{lemma:intersection}, where $\gamma$ is a perturbation hyperparameter and $M$ is a natural number representing the embedding space dimensionality. 

At the end of the training, we have a sequence $([\barbelow e_1, \bar e_1], ..., [\barbelow e_T, \bar e_T])$ of interval embeddings dedicated to consecutive CL tasks. 
The above procedure forces such segments to have non-empty intersections as it is shown in Lemma \ref{lemma:intersection}. Therefore, we can define a universal embedding as 
$[\barbelow e, \bar e] = \bigcap_t \ \left[\barbelow e_t, \bar e_t\right]$.
As mentioned above, when $\epsilon^*$ is trainable, finding a universal embedding is not guaranteed.
\begin{lemma}\label{lemma:intersection}
    Let $(e_1, e_2, \ldots, e_T)$ be embedding centers, $T$ be the number of CL tasks, $\gamma > 0$ be a perturbation value, %$\epsilon^*$ be a vector of ones, 
    $\hypnet(\cdot; \hypwei)$ be a hypernetwork with weights $\hypwei$, and $M$ be a natural number representing the dimensionality of the embedding space, $\barbelow e_t$, $\bar e_t$ be calculated according to \eqref{preem2em} with $\epsilon_t \equiv \epsilon^*$ being a vector of ones, where $t \in \{ 1, 2, \ldots, T \}$. 
    Then
    $$
    [\barbelow{e}, \bar{e}] = \bigcap_{t=1}^T [\barbelow{e}_t, \bar{e}_t]
    $$
    has a non-empty intersection.
\end{lemma}
The proof of Lemma \ref{lemma:intersection} is presented in Appendix~\ref{sec:proof}. Elements from the intersection allow one to solve multiple tasks simultaneously. 
% \kamil{Trudne zdanie i dalej uważam, że może być nieco mylące. W dodatku co znaczy z "some accuracy"?}
% \KamilSide{Czy doprecyzowujemy, że chodzi o zachowanie na zbiorach treningowych/walidacyjnych, a nie testowych? ODP.: Tak, dodamy to do limitations} 
Then, we can use the center of the universal embedding and the trained hypernetwork to produce a single target network. In the evaluation, it is sufficient to use such prepared target weights, and we no longer need to store the hypernetwork and trained interval embeddings. 

% In such a training procedure we use Hypernetwork as the Meta-Trainer which is used in training, but is not used in the evaluation procedure.
% Below we specify the conditions of our proposed architecture to be non-forgetting. Under the additional condition of $\epsilon_t$ in \eqref{preem2em} being fixed, there is always a non-empty intersection of embedding intervals for different tasks that assure an earlier trained performance of the target network in these tasks.
\paragraph{Guarantees of non-forgetting} 
Below, we specify the conditions of our proposed architecture to be non-forgetting. When the intersection of embedding intervals is non-empty, and the regularization for the hypernetwork training is effective, then we achieve non-forgetting, as the theorem below specifies.  %we can formulate Theorem \ref{theorem:non_forgetting}.

\begin{theorem}\label{theorem:non_forgetting}
    Let $\left( e_1,\ldots, e_T \right)$ be embedding centers with corresponding perturbation vectors $\left( \epsilon_1, \ldots, \epsilon_T \right)$, $T$ be the number of CL tasks, $D_t = \left(X_t, Y_t\right)$ be a pair of observations $X_t$ and their corresponding one-hot encoded classes $Y_t$ taken from the $t$-th task, $\hypnet(\cdot; \hypwei_T)$ be a hypernetwork with weights $\hypwei_T$ obtained at the end of training the $T$-th task. Let also $\phi\left( \cdot; \hypnet\left( [\barbelow e_t, \bar e_t]; \eta_T \right) \right)$ be a target network with interval weights produced by the hypernetwork such that for every $\epsilon > 0$, $t\in \lbrace 1, 2, \ldots, T\rbrace$, $e_t$, $\epsilon_t$, and $x\in X_t$, there exists $y \in Y_t$ such that
    $$
    \sup_{\mu \in \left[ \munderbar{e}_t, \bar{e}_t \right]}\norm{}{y - \phi\left(x, \hypnet\left(\mu; \hypwei_T\right)\right)}_2 \leq \epsilon.
    $$

    Assume also that 
    $
    \bigcap_{t=1}^T\left[\munderbar{e}_t, \bar{e}_t \right]
    $ 
    is non-empty and let us introduce
    \begin{align*}
    A_t = \lbrace \mu | & \forall_{\epsilon > 0}\forall_{x \in X_t}\exists_{y \in Y_t} 
    \\& 
    \sup_{\mu \in \left[ \munderbar{e}_t, \bar{e}_t \right]}\norm{}{y - \phi\left(x, \hypnet\left(\mu; \hypwei_T\right)\right)}_2 \leq \epsilon\rbrace, \\ 
    & t\in\lbrace 1, 2, \ldots, T\rbrace, 
    \end{align*}
    \begin{align*}
    A = \lbrace \mu | & \forall_{\epsilon > 0}\forall_{x \in X_t}\exists_{y \in Y_t} 
    \\& 
    \sup_{\mu \in \bigcap_{t=1}^T\left[\munderbar{e}_t, \bar{e}_t \right]}\norm{}{y - \phi\left(x, \hypnet\left(\mu; \hypwei_T\right)\right)}_2 \leq \epsilon\rbrace.
    \end{align*}
    
    Then, we have guarantees of non-forgetting within the region $[\barbelow{e}, \bar{e}] = \bigcap_{t=1}^T [\barbelow{e}_t, \bar{e}_t]$, i.e. $A \subset A_t$, for each $t\in \lbrace 1, 2, \ldots, T\rbrace$.
\end{theorem}

The proof of the above theorem is  in Appendix~\ref{sec:proof}. 

% Such a formulation allows us to get guarantees for each $e \in [\barbelow e_t, \bar e_t]$

\begin{figure}
    \centering

        \includegraphics[width=0.49\textwidth]{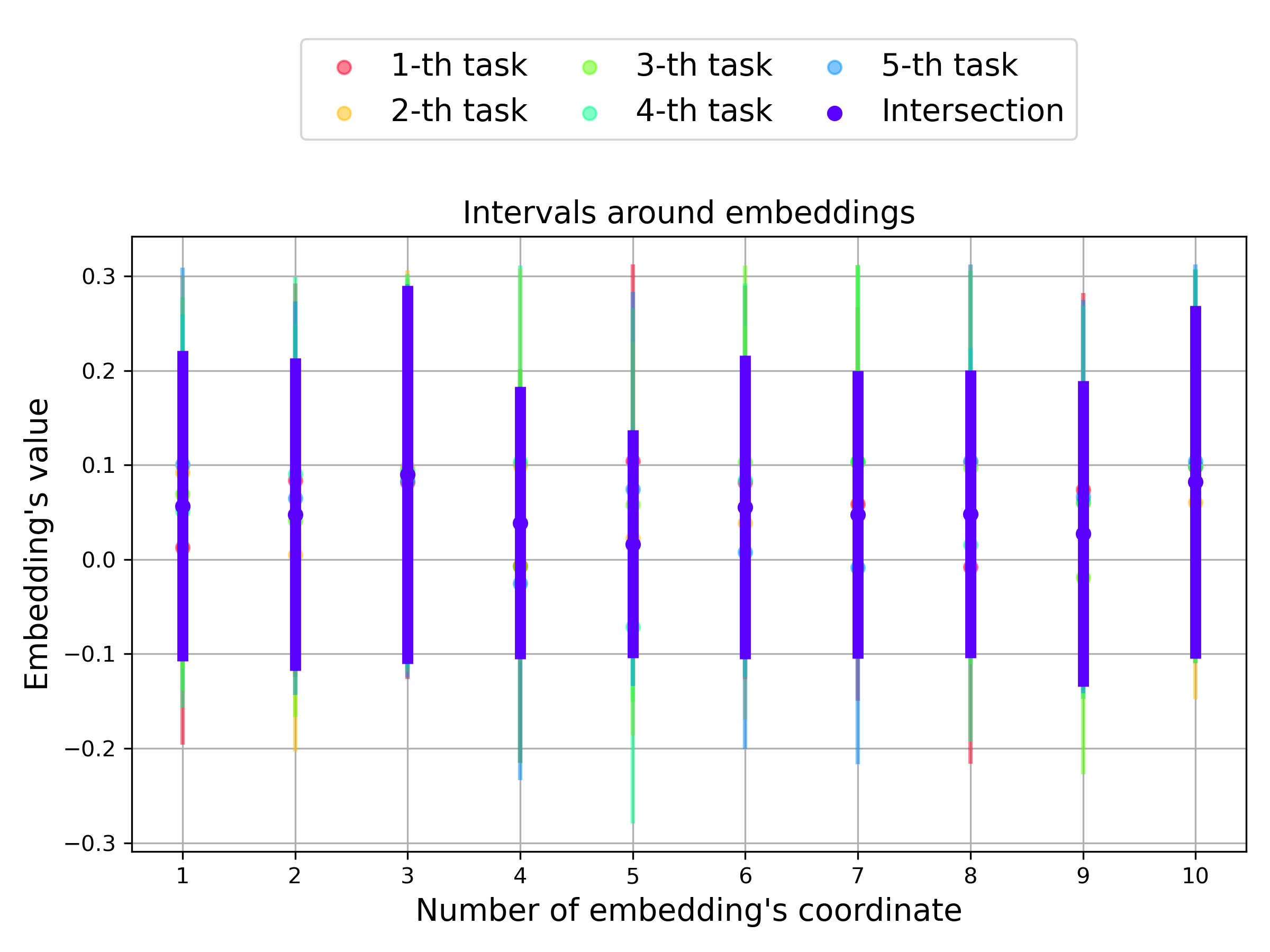}
        \caption{Embedding intervals for Split CIFAR-100, 5~tasks with 20 classes each, using the $\cos\left(\cdot\right)$ nesting method. The ten first embedding coordinates are shown.}\label{fig:interval_embeddings_cifar100}
\end{figure}

% \begin{wraptable}{r}{5.5cm}
\begin{table}[!h]
% \begin{wraptable}{l}{8.cm}
% \small
\caption{Average accuracy with a standard deviation of different continual learning methods in the TIL setup. Results for different methods than \our{} are derived from other papers.}
\centering
{
\fontsize{9pt}{9pt}\selectfont
\begin{tabular}{@{}l@{\;}c@{}c@{}c@{\;}c@{}}
    \toprule
    % \textbf{Method} & \textbf{Per. MNIST} & \textbf{Split MNIST} & \textbf{Split CIFAR-100} & \textbf{TinyImageNet} \\

    \textbf{Method} & \textbf{Permuted } & \textbf{Split} & \textbf{Split} & \textbf{Tiny}\\
    & \textbf{MNIST} & \textbf{MNIST} & \textbf{CIFAR-100} & \textbf{ ImageNet}\\     
    
    \midrule
    HAT & $97.67 \spm 0.02$ & $-$ & $72.06 \spm 0.50$ & $-$ \\
    GPM & $94.96 \spm 0.07$ & $-$ & $73.18 \spm 0.52$ & $67.39 \spm 0.47$ \\
    PackNet & $96.37 \spm 0.04$ & $-$ & $72.39 \spm 0.37$ & $55.46 \spm 1.22$ \\
    SupSup & $96.31 \spm 0.09$ & $-$ & $75.47 \spm 0.30$ & $59.60 \spm 1.05$ \\
    La-MaML & $-$ & $-$ & $71.37 \spm 0.67$ & $66.99 \spm 1.65$ \\
    FS-DGPM & $-$ & $-$ & $74.33 \spm 0.31$ & $70.41 \spm 1.30$ \\
    WSN, $30\%$ & $96.41 \spm 0.07$ & $-$ & $75.98 \spm 0.68$ & $\mathbf{70.92 \spm 1.37}$ \\
    WSN, $50\%$ & $96.24 \spm 0.11$ & $-$ & $76.38 \spm 0.34$ & $69.06 \spm 0.82$ \\
    \midrule
    EWC & $95.96 \spm 0.06$ & $99.12 \spm 0.11$ & $72.77 \spm 0.45$ & $-$ \\
    SI & $94.75 \spm 0.14$ & $99.09 \spm 0.15$ & $-$ & $-$ \\
    DGR & $97.51 \spm 0.01$ & $99.61 \spm 0.02$ & $-$ & $-$ \\
    HNET & $97.57 \spm 0.02$ & $\mathbf{99.79 \spm 0.01}$ & $-$ & $-$ \\
    \midrule
    \our{} & $\mathbf{97.78 \spm 0.09}$ & $ 99.75 \spm 0.08$ & $\mathbf{77.46 \spm 1.34}$ & $66.10 \spm 0.62$ \\
    \bottomrule
\end{tabular}
}
\label{tab:continual}
\end{table}
% \end{wraptable}

% \begin{wraptable}{l}{7.5cm}
\begin{table}[t]
    \centering
    \caption{CL in the CIL setup. The task identity results for entropy approach
    Last -- last task accuracy (max), Avg. -- average task accuracy (max). The best results are indicated with {\bf bold}. If standard deviations are provided, the results represent the average of 5 runs.
    }
    \fontsize{9pt}{9pt}\selectfont
    \setlength{\tabcolsep}{4pt} % Adjust column separation
    \begin{tabular}{@{}c@{\hskip 0.1in}c@{\hskip 0.1in}c@{}cl@{}}
        \toprule
        \textbf{Method} & \textbf{Permuted} & \textbf{Split} & \multicolumn{2}{c}{\textbf{Split CIFAR-100}} \\
        \cmidrule(lr){4-5}
         & \raisebox{1ex}[1ex]{\textbf{MNIST}} & \raisebox{1ex}[1ex]{\textbf{MNIST}} & \textbf{\quad\;Last\;\quad} & \textbf{Avg.} \\
        \midrule
        HNET+ENT & $91.75 \spm 0.21$ & $69.48 \spm 0.80$ & - & - \\
        EWC & $33.88 \spm 0.49$ & $19.96 \spm 0.07$ & - & - \\
        SI & $29.31 \spm 0.62$ & $19.99 \spm 0.06$ & - & - \\
        DGR & $\mathbf{96.38 \spm 0.03}$ & $\mathbf{91.79 \spm 0.32}$ & - & - \\
        \midrule
        Euclidean-NCM & - & - & $30.6$ & $50.0$ \\
        FeTrIL & - & - & $46.2$ & $61.3$  \\
        FeCAM & - & - & $48.1$ & $62.3$  \\
        DS-AL & - & - & - & $\mathbf{68.39 \spm 0.16}$ \\
        EFC & - & - & - & $65.97 \spm 1.19$ \\
        \midrule
        \our{} & $94.60 \spm 1.11$ & $55.38 \spm 4.89$ & $44.58$ & $43.45 \spm 0.97$ \\
        \bottomrule
     \end{tabular}
    \label{tab:CL3}
\end{table}
% \end{wraptable}

\begin{table}[!h]
\caption{Average test accuracy with a standard deviation of the InterContiNet and \our{} methods for TIL, DIL, and CIL scenarios. Results for the DIL and CIL scenario using \our{} are calculated with the universal embedding and entropy method, respectively. Results for InterContiNet are derived from \cite{wolczyk2022continual}. The standard deviations for Split MNIST and Split CIFAR-10 are calculated over 5 runs. The standard deviation for Split CIFAR-100 in TIL setup is calculated over 3 runs. The best results are indicated with {\bf bold}.}
\centering
{
\scriptsize
\begin{tabular}{@{}ll@{\;}c@{\hskip 0.1in}c@{\hskip 0.1in}c@{}}
    \toprule
     & \textbf{Method} & \textbf{Split MNIST} & \textbf{Split CIFAR-10} & \textbf{Split CIFAR-100} \\
    \midrule
    \multirow{2}{*}{TIL} & InterContiNet & $98.93 \spm 0.05$ & $ 72.64 \spm 1.18 $ & $42.0 \spm 0.2$  \\
    & \our{} & $\mathbf{99.75 \spm 0.08}$ & $ \mathbf{90.91 \spm 0.95} $ & $\mathbf{79.23 \spm 0.36}$ \\
    \midrule
    \midrule
    \multirow{2}{*}{CIL} & InterContiNet & $40.73 \spm 3.26$ & $ 19.07 \spm 0.15 $ & $-$  \\
    & \our{}  & $\mathbf{55.37 \spm 4.24}$ & $\mathbf{24.19 \pm 1.12} $ & $-$ \\
    \midrule
    \midrule
    \multirow{2}{*}{DIL} & InterContiNet & $77.77 \spm 1.24$ & $ 69.48 \spm 1.36 $ & $-$  \\
    & \our{}  & $\mathbf{79.29 \spm 3.97}$ & $ \mathbf{74.01 \spm 1.36} $ & $14.22$ \\
    \bottomrule
\end{tabular}
}
% \vspace{-0.3cm}
\label{tab:intercontinet_comparison}
% \end{wraptable}
\end{table}

%%%%%%%%%%%%%%%%%%%%%%%%%%%%%%%%%%%%%%%%
\section{Experiments}
%%%%%%%%%%%%%%%%%%%%%%%%%%%%%%%%%%%%%%%%

In this section, we provide an overview of the results of our method under different training assumptions. We cover a variety of incremental learning setups to ensure a broad analysis of the interval arithmetic approach in CL. Moreover, we show the best results obtained on each dataset using our training method.

% \begin{wraptable}{l}{8cm}
\begin{table}[!h]
\caption{Average test accuracy with a standard deviation of the InterContiNet and \our{} methods, when the task identity is known or unknown, respectively. Results for InterContiNet are derived from \cite{wolczyk2022continual}. The standard deviations for Split MNIST and Split CIFAR-10 are calculated over 5 seeds.The standard deviation for CIFAR-100 is calculated over 3 seeds.}
\centering
{
\scriptsize
\begin{tabular}{@{}ll@{\;}c@{\;}c@{\;}c@{}}
    \toprule
     & \textbf{Method} & \textbf{Split MNIST} & \textbf{Split CIFAR-10} & \textbf{Split CIFAR-100} \\
    \midrule
    \multirow{2}{*}{TIL} & InterContiNet & $98.93 | 0.05$ & $ 72.64 | 1.18 $ & $42 | 0.2$  \\
    & \our{} & $\mathbf{99.75 | 0.08}$ & $ \mathbf{90.91 | 0.95} $ & $\mathbf{79.23 | 0.36}$ \\
    \midrule
    \midrule
    \multirow{2}{*}{CIL} & InterContiNet & $40.73 | 3.26$ & $ 19.07 | 0.15 $ & $-$  \\
    & \our{}  & $\mathbf{78.8 | 5.39}$ & $ \mathbf{72.87 | 2.2} $ & $-$ \\
    \bottomrule
\end{tabular}
}
% \vspace{-0.3cm}
\label{tab:intercontinet_comparison}
% \end{wraptable}
\end{table}

\paragraph{Training setup}
We apply three typical CL training setups: TIL, in which the task identity of test samples is known, Domain Incremental Learning (DIL), and CIL. In the last two setups, the task identity during inference is unknown. 
When we consider \our{} in TIL, we do not use any nesting approaches, and the input embedding intervals can be of different lengths. In CIL we use entropy to determine task ID during the test phase, and in DIL, we use the $\cos(\cdot)$ nesting method for the embeddings, and the input intervals are of the same length, to prevent the embeddings from collapsing to a trivial point.
A~more detailed description of these setups can be found in Appendix \ref{sec:train-details}.

\paragraph{Datasets} We conduct experiments on 5 publicly available datasets: Permuted MNIST \cite{oshg2019hypercl}, Split MNIST \cite{oshg2019hypercl}, Split CIFAR-10 \cite{wolczyk2022continual}, Split CIFAR-100 \cite{goswami2024fecam} and TinyImageNet \cite{goswami2024fecam}. We~encourage the reader to proceed to the supplementary materials for the details about the task division.

\paragraph{Architectures and baselines}

As target networks for Permuted MNIST and Split MNIST, we use two-layered MLPs with 1000 neurons per layer for the first dataset and 400 neurons per layer for the second one. For Split CIFAR-100 with 5 and 10 tasks, each with equally distributed labels per task, and for TinyImageNet, we select a convolutional network, specifically the ResNet-18, as in \cite{goswami2024fecam}. To ensure a fair comparison with the InterContiNet method, we also train the AlexNet architecture for Split CIFAR-10 and Split CIFAR-100  (20 tasks, each with 5 labels). We modify the AlexNet architecture according to \cite{wolczyk2022continual}. Specifically, we add batch normalization after each convolutional and fully connected layer. However, we do not include batch normalization after the final linear layer, as it serves as the classification layer. Due to limited GPU resources, the number of neurons in the first two fully connected layers is reduced by half, compared to the version of AlexNet used in \cite{wolczyk2022continual}. In all cases, the hypernetwork is an MLP with either one or two hidden layers. Whenever a convolutional target network is used, we apply interval relaxation during training. The description of this approach can be found in the Appendix \ref{sec:relax_tech}.
We compare our solution with InterContiNet~\cite{wolczyk2022continual}, WSN~\cite{kang2023forgetfree}, HNET~\cite{von2019continual}, FeCAM~\cite{goswami2024fecam}, as well as some other strong CL approaches mentioned in \cite{goswami2024fecam}. The choice of methods used for comparison depends on the applied training setup.
% We~use different model parameters, depending on the choice of the training setup, as well as some training relaxation techniques, depending on the target network architecture.
% We~encourage the reader to proceed to the supplementary materials section for a~more detailed description.
% More details can be found in the supplementary materials.

% \paragraph{Baselines}
% .

\paragraph{Experimental results} Results for the known task identity setup are presented in Table \ref{tab:continual}. Our method outperforms its competitors on Permuted MNIST and Split CIFAR-100 datasets while reaching the second-best result on the Split MNIST dataset. Moreover, on TinyImageNet, we obtained stable results, meaning the second-lowest standard deviation score.
For the same training setup, the comparison between InterContiNet and \our{} is shown in Table \ref{tab:intercontinet_comparison}. We obtained better results on all three datasets used, specifically showing the advantage of \our{} in case of training convolutional networks. In this table, results for Split CIFAR-100 are obtained and compared using the InterContiNet setup: 20 tasks with 5 classes each. 

Results for the unknown task identity setup are presented in Tables \ref{tab:CL3} and \ref{tab:intercontinet_comparison}. Despite the large standard deviation on Split MNIST, we obtain one universal embedding which solves all tasks at once, as we do for Permuted MNIST. On Split CIFAR-100 we obtain the best maximum last task accuracy. Moreover, compared to other methods, \our{} achieves a smaller deviation between the last task and average task accuracy, showing consistency in consecutive task results. Unfortunately, the universal embedding found by \our{} performs poorly on the Split CIFAR-100 dataset, achieving only approximately $15\%$ accuracy. We argue that this is due to the larger number of classes per task, which makes it challenging to identify a single universal embedding capable of solving such tasks simultaneously. As in the previous setup, \our{} method outperforms InterContiNet on all datasets by a large margin, which is shown in Table \ref{tab:intercontinet_comparison}. In Figure \ref{fig:interval_embeddings_cifar100} we present the interval embeddings for each of 5 tasks of Split CIFAR-100. It is shown that a non-empty universal embedding exists, moreover, it does not collapse to a point in the embedding space.

In~Appendix~\ref{sec:ablation} we present a~study focused on interval lengths, interval nesting and regularization. In~Appendix~\ref{sec:additional-exp-results} we present more detailed insight of the above experimental results.

\section{Conclusions and limitations}

In this paper, we introduce \our{}, a continual learning architecture that employs interval arithmetic in the trained neural model and a hypernetwork producing its weights.  
\our{} uses interval embeddings for consecutive tasks and trains the hypernetwork to transform these embeddings into weights of the target network. The proposed mechanism allows us to train interval networks on large datasets in continual learning scenarios, in the TIL, CIL, and DIL setups. \our{} enables the generation of a universal embedding thanks to interval arithmetic and hypernetwork training. The intersection of intervals, i.e. a universal embedding, can solve all tasks simultaneously. In such a scenario, the hypernetwork functions solely as a {\em meta-trainer}, meaning that we maintain only a single set of weights generated by the hypernetwork through the universal embedding. This approach significantly reduces memory usage. Furthermore, we provide formal guarantees of non-forgetting.

\paragraph{Limitations}
% Our \our{} dramatically increases the performances of InterContiNet, but we have a problem with convolutional architectures. It is an open question of how to efficiently train Interval Arithmetic on convolutional layers.  
% There are architectural components for which there are no obvious interval counterparts. A~notable example of these components is two-dimensional batch normalization.
% Our approach has several limitations. 
Non-forgetting guarantees within the interval parameter space generated by the hypernetwork are valid only as long as the hypernetwork's regularization term remains effective. 
Second, we have observed that achieving satisfactory performance becomes challenging when the number of classes in a given task is large. Splitting such a task into subtasks may be advantageous and subsequently finding a universal embedding, but we leave this for future work.

\appendix

\section{Interval arithmetic in neural networks}
\label{sec:ia-in-nns}

Suppose $A,B \subset \R$ represent intervals. For all $\bar a, \barbelow a, \bar b, \barbelow b \in \R$, where $A = [\barbelow a, \bar a]$ and $B = [\barbelow b, \bar b]$, arithmetic operations can be defined as in \cite{lee2004first}:
\begin{itemize}
    \item addition:
    $
     [\barbelow a, \bar a] +  [\barbelow b, \bar b] = [\barbelow a + \barbelow b, \bar a + \bar b]
    $
    \item multiplication:
    $
     [\barbelow a, \bar a] *  [\barbelow b, \bar b] = [\min( \barbelow a * \barbelow b, \barbelow a * \bar b, \bar a * \barbelow b, \bar a * \bar b),
     \max( \barbelow a * \barbelow b, \barbelow a * \bar b, \bar a * \barbelow b, \bar a * \bar b )
     ]
    $      
\end{itemize}
Therefore, interval arithmetic is capable of executing affine transformations, enabling the implementation of both fully-connected and convolutional layers in neural networks.

We denote by $\phi(x;\theta)$ a layered neural classifier with input $x$ and weights $\theta$. It comprises a series of transformations 
\begin{align}
\phi\left(x;\theta\right) 
&= \left(h_L \circ h_{L-1} \circ \ldots h_1\right) \left(x; \theta\right) \\ 
&= h_L( h_{L-1}( \ldots h_1(x)  ) ).
\end{align}
All component transformations, $h_l$, are also based on the weights $\theta$, which we omit in the notation. The final output $\phi\left(x; \theta\right) = z_L \in \mathbb{R}^N$ is to indicate one of $N$ classes to which the input $x$ belongs.

In InterContiNet \cite{wolczyk2022continual}, weight matrices and bias vectors of each layer $h_l$ are located in hyperrectangles $[\barbelow W_l, \bar W_l]$ and $[\barbelow b_l, \bar b_l]$, respectively. We denote 
$[\barbelow \theta, \bar\theta] = \langle [\barbelow W_1, \bar W_1], [\barbelow b_1, \bar b_1], \dots, [\barbelow W_L, \bar W_L], [\barbelow b_L, \bar b_L]\rangle$ to represent intervals of all the trainable parameters in the network.

The transformation in $l$-th layer, $l=1,\dots,L$, of InterContiNet can be expressed as: 
\begin{equation}
[\barbelow z_l, \bar z_l] = h_l([\barbelow z_{l-1}, \bar z_{l-1}]) 
\end{equation}
where $z_0 = [x,x]$ is the input, $[\barbelow z_l, \bar z_l]$ is the hyperrectangle of $l$-th layer activations, 
\begin{align}
h_l([\barbelow z_{l-1},\bar z_{l-1}]) 
&= [ \barbelow h_l, \bar h_l]([\barbelow z_{l-1},\bar z_{l-1}]) \\ 
&= \psi([  \barbelow W_l , \bar W_l ] [\barbelow z_{l-1}, \bar z_{l-1}] + [\barbelow b_l, \bar b_l]),  
\end{align}
with $\psi$ being an activation function with positive output, and 
\begin{align}\label{z_l:min:max}
\underline{z}_{l} &= \min_{\underline{z}_{l-1} \leq z_{l-1} \leq \overline{z}_{l-1} }h_{l}(z_{l-1}), \\ \overline{z}_{l} &= \max_{\underline{z}_{l-1} \leq z_{l-1} \leq \overline{z}_{l-1}}h_{l}({z}_{l-1}).
\end{align}
Monotonicity of $\psi$ (e.g., ReLU, logistic sigmoid) is enough for efficient computation of $[\barbelow z_l, \bar z_l]$. % \eqref{z_l:min:max}. 

The input $z_0 = [x,x]$ is assumed to be nonnegative, $x\geq0$. The final output 
$
\phi(z_0;[\barbelow \theta_1, \bar \theta_L ]) = h_L( h_{L-1}( \ldots h_1(z_0)  ) ) = [\barbelow z_L, \bar z_L]
$ 
comprises $N$ intervals, one for each class.

\section{Layered hypernetwork with interval embedding input}
\label{sec:hypernet-with-em-in}

Let the hypernetwork $\hypnet([\barbelow{e_t}, \bar{e_t}]; \hypwei)$ comprise a series of transformations $[\barbelow h_1, \bar h_1](\cdot), \dots, [\barbelow h_L, \bar h_L](\cdot)$ across its $L$ layers. 
The final output $[\barbelow h_L, \bar h_L]\left([\barbelow h_{L-1}, \bar h_{L-1}]\left(\ldots [\barbelow h_1, \bar h_1]([x, x])\right)\right) = [\barbelow z_L, \bar z_L]$ produces interval weights for the target model. This result represents a composition of weight intervals of consecutive layers.

In \our{}, we propagate the interval embedding
$[\barbelow z_0, \bar z_0] = [\barbelow e, \bar e]$ from the input to the output layer. 
The output bounding box $[\barbelow{z}_{l}, \bar{z}_{l}]$ from the hypernetwork $\hypnet([\barbelow{e_t}, \bar{e_t}]; \hypwei)$ is calculated in the following way:
% By applying the transformation $[\barbelow h_k, \bar h_k]([\batbelow z_{k-1}, \bar z_{k-1}]) = \psi([\barbelow W_k, \bar W_k] [\barbelow z_{k-1} \bar z_{k-1}] + [\barbelow b_k, \bar b_k])$ 
% to the interval $[\underline{z}_{k-1}, \overline{z}_{k-1}]$
$$
\begin{array}{lll}
    & \mu_{l-1} = \frac{\overline{z}_{l-1} + \underline{z}_{l-1}}{2},
    & r_{l-1} = \frac{\overline{z}_{l-1} - \underline{z}_{l-1}}{2},\\[0.8ex]
    & \mu_l = W_l \mu_{l-1} + b_l,
    & r_l = |W_l| r_{l-1},\\[0.8ex]
    & \underline{z}_l = \psi\left(\mu_l - r_l\right),
    & \overline{z}_l = \psi\left(\mu_l + r_l\right),\\
\end{array}
$$
where $W_l$ and $b_l$ are $l$-th layer weight matrix and bias vector, respectively, $\eta = \langle W_1, b_1, W_2, b_2\ldots, W_L, b_L \rangle$, and $|\cdot|$ is an element-wise absolute value operator. 

%For a monotonic activation function $h_k$, we get the interval bound defined by:
%$$
%      \underline{z}_k = h(\underline{z}_{k-1}) \text{ , }
%      \overline{z}_k = h(\overline{z}_{k-1}).
%$$

Therefore, our interval hypernetwork  $\hypnet([\barbelow e_t, \bar e_t]; \hypwei) = [\barbelow \theta_t, \bar \theta_t]$ transforms the embedding of the $t$-th task into interval weights of the target network.
% Also, using a hypernetwork, we can propagate a point $e$ instead of an interval with a trivial segment $[\barbelow e,\bar e]$. We will use the notation
%$\hypnet(e, \hypwei) = \hypnet([\barbelow e, \bar e], \hypwei) = [\barbelow \theta, \bar \theta] = [\theta,\theta] = \theta$. 
\paragraph{\our{}}

In \our{}, trainable interval embeddings $[\barbelow e_t, \bar e_t]$ are used for producing separated network weights for each continuous learning task $t$, $t \in \lbrace 1, ..., T \rbrace$. These embeddings are propagated through the trainable IBP-based hypernetwork
$$
\hypnet([\barbelow e_t, \bar e_t]; \hypwei ) = [\barbelow \theta_t, \bar \theta_t]
$$ 
into interval weights of the target model. To derive the equations for linear layers, we introduce the following terms:
\begin{align*}
    \munderbar{W}_{k-1}^{+} &= \max{\lbrace\munderbar{W}_{k-1}, 0\rbrace}, \quad
    & \munderbar{W}_{k-1}^{-} = \max{\lbrace-\munderbar{W}_{k-1}, 0\rbrace}, \\
    \bar{W}_{k-1}^{+} & = \max{\lbrace\bar{W}_{k-1}, 0\rbrace}, 
    & \bar{W}_{k-1}^{-} = \max{\lbrace -\bar{W}_{k-1}, 0\rbrace},
\end{align*}

where $k\in\lbrace 2, \ldots, L\rbrace$, $[\barbelow \theta_t, \bar \theta_t] = \langle \left[ \munderbar{W}_1, \bar{W}_1 \right],  \left[ \munderbar{b}_1, \bar{b}_1 \right], \ldots, \left[ \munderbar{W}_L, \bar{W}_L \right], \left[ \munderbar{b}_L, \bar{b}_L \right]\rangle$,
$\min{\lbrace\cdot, \cdot\rbrace}$ and $\max{\lbrace\cdot, \cdot\rbrace}$ denote element-wise minimum and maximum operations. Then, one need to calculate such an expression:
\begin{align*}
    \left[\underline{z}_k, \overline{z}_k\right] &= \psi\left(\left[ \munderbar{W}_{k-1}, \bar{W}_{k-1} \right]\cdot \left[ \underline{z}_{k-1}, \overline{z}_{k-1} \right] + \left[\underline{b}_{k-1}, \overline{b}_{k-1}\right]\right).
\end{align*}

Based on the definition of multiplying two intervals, we obtain that:
\begin{align*}
    &\left[ \munderbar{W}_{k-1}, \bar{W}_{k-1} \right] \cdot \left[ \underline{z}_{k-1}, \overline{z}_{k-1} \right] \\ &= \left[\munderbar{W}_{k-1}^{+} - \munderbar{W}_{k-1}^{-}, \bar{W}_{k-1}^{+} - \bar{W}_{k-1}^{-}  \right]\cdot \left[ \underline{z}_{k-1}, \overline{z}_{k-1} \right] \\ &= \left( \left[\munderbar{W}_{k-1}^{+}, \bar{W}_{k-1}^{+}\right] - \left[\bar{W}_{k-1}^{-},  \munderbar{W}_{k-1}^{-}  \right]\right)\cdot \left[ \underline{z}_{k-1}, \overline{z}_{k-1} \right] \\ &= \left[\munderbar{W}_{k-1}^{+}, \bar{W}_{k-1}^{+}\right]\cdot \left[ \underline{z}_{k-1}, \overline{z}_{k-1} \right] \\ &-  \left[\bar{W}_{k-1}^{-}, \munderbar{W}_{k-1}^{-}\right]\cdot \left[ \underline{z}_{k-1}, \overline{z}_{k-1} \right].
\end{align*}
Generally, distributivity does not hold for interval arithmetic, which means that $\left(\left[ a,b \right] + \left[ c,d \right]\right) \cdot \left[ e,f \right]$ does not equal $\left[ a+c,b+d \right]\cdot \left[ e,f \right]$ for some $a,b,c,d,e,f \in \mathbb{R}$ in general. However, this property holds when the multiplication of intervals $\left[a,b\right]\cdot \left[ c,d \right]$ results in an interval with non-negative endpoints. This is why the decomposition of lower and upper matrices into the subtraction of positive and negative matrices is useful. Using the definition of interval multiplication and the fact that $\munderbar{W}_{k-1}^{+}, \bar{W}_{k-1}^{+}, \barbelow{W}_{k-1}^{-}, \bar{W}_{k-1}^{-}, \underline{z}_{k-1}, \overline{z}_{k-1} \geq ~ 0$, we have:
\begin{align*}
    &\left[\munderbar{W}_{k-1}^{+}, \bar{W}_{k-1}^{+}\right] \cdot \left[ \underline{z}_{k-1}, \overline{z}_{k-1} \right] \\
    &= \left[\min_{\substack{a\in\lbrace \munderbar{W}_{k-1}^{+}, \bar{W}_{k-1}^{+}\rbrace \\ b\in\lbrace \underline{z}_{k-1}, \overline{z}_{k-1}\rbrace}} \lbrace a \cdot b \rbrace, \max_{\substack{a\in\lbrace \munderbar{W}_{k-1}^{+}, \bar{W}_{k-1}^{+}\rbrace \\ b\in\lbrace \underline{z}_{k-1}, \overline{z}_{k-1}\rbrace}} \lbrace a \cdot b \rbrace \right] \\ &= \left[ \munderbar{W}_{k-1}^{+}\cdot \underline{z}_{k-1}, \bar{W}_{k-1}^{+}\cdot \overline{z}_{k-1} \right].
\end{align*}

Analogously,
\begin{align*}
     &\left[\bar{W}_{k-1}^{-}, \munderbar{W}_{k-1}^{-}\right] \cdot \left[ \underline{z}_{k-1}, \overline{z}_{k-1} \right] \\
    &= \left[\min_{\substack{a\in\lbrace \munderbar{W}_{k-1}^{-}, \bar{W}_{k-1}^{-}\rbrace \\ b\in\lbrace \underline{z}_{k-1}, \overline{z}_{k-1}\rbrace}} \lbrace a \cdot b \rbrace, \max_{\substack{a\in\lbrace \munderbar{W}_{k-1}^{-}, \bar{W}_{k-1}^{-}\rbrace \\ b\in\lbrace \underline{z}_{k-1}, \overline{z}_{k-1}\rbrace}} \lbrace a \cdot b \rbrace \right] \\ &= \left[ \bar{W}_{k-1}^{-}\cdot \underline{z}_{k-1}, \munderbar{W}_{k-1}^{-}\cdot \overline{z}_{k-1} \right] \\ &= -\left[ -\munderbar{W}_{k-1}^{-}\cdot \overline{z}_{k-1}, -\bar{W}_{k-1}^{-}\cdot \underline{z}_{k-1} \right] ,
\end{align*}
which gives us the final form of the equations:
\begin{align*}
    \underline{z}_k &= \psi\left(\munderbar{W}_{k-1}^{+}\cdot \underline{z}_{k-1} - \munderbar{W}_{k-1}^{-}\cdot\overline{z}_{k-1} + \underline{b}_{k-1} \right), \\
    \overline{z}_k &= \psi\left(\bar{W}_{k-1}^{+}\cdot \overline{z}_{k-1} - \bar{W}_{k-1}^{-}\cdot\underline{z}_{k-1} + \overline{b}_{k-1} \right).
\end{align*}
Since the equations above are derived for propagation through fully connected layers, a similar version for convolutional layers can be easily obtained, as a convolution operation can be expressed as a multiplication operation. The final form of these equations is as follows:
\begin{align*}
    \underline{z}_k &= \psi\left(\munderbar{W}_{k-1}^{+}* \underline{z}_{k-1} - \munderbar{W}_{k-1}^{-}*\overline{z}_{k-1} + \underline{b}_{k-1} \right), \\
    \overline{z}_k &= \psi\left(\bar{W}_{k-1}^{+}* \overline{z}_{k-1} - \bar{W}_{k-1}^{-}*\underline{z}_{k-1} + \overline{b}_{k-1} \right),
\end{align*}
where the $*$ operator denotes the convolution operation. Unfortunately, such equations cannot be easily obtained for batch normalization layers.
\section{The \our{} algorithm}
\label{sec:algorithm}

\our{} is presented below as Algorithm~\ref{alg:hint}. 

    \begin{algorithm}[!t]
    \caption{The pseudocode of \our{}.}
    \label{alg:hint}
    \begin{algorithmic}
    \Require 
    %{\bf Require:} 
    hypernetwork $\hypnet$ with weights $\hypwei$, target network $\phi$, softmax function $\sigma(\cdot)$, perturbation value $\gamma > 0$, regularization strength $\beta > 0$, cross-entropy components strength
    % parameter controlling the convex combination of standard cross-entropy and worst-case cross-entropy
    $\kappa \in \left(0, 1\right)$, dimensionality of the embedding space $M$, $n$ training iterations, datasets $\lbrace D_1, D_2, ..., D_T\rbrace$, 
    $(x_{i,t}, y_{i,t}) \in D_{t}, t \in \lbrace 1, \ldots, T \rbrace$, $i\in \lbrace 1, \ldots, N_t \rbrace$, $N_t$ is a number of samples in the dataset $D_t$.

    \Ensure 
    %{\bf Ensure:} 
    updated hypernetwork $\hypnet$ weights $\hypwei$
    
    %\BlankLine
    \State 
    Initialize randomly weights $\hypwei$ with pre-embeddings $(\preem_1, \preem_2, ..., \preem_T)$ and corresponding perturbation vectors $(\epsilon^{*}_1, \epsilon^{*}_2, \ldots, \epsilon^{*}_T)$

    \For{$t \leftarrow 1$ to $T$} 
        \If{$t > 1$}  
            \State 
            $\hypwei^{*} \leftarrow  \hypwei$

            \For{$t^{\prime} \leftarrow 1$ to $t-1$}
                \State
                $\epsilon^{*}_{t^{\prime}} \leftarrow \gamma\cdot\sigma\left( \epsilon^{*}_{t^{\prime}} \right)$
                \State
                Store $[\boldsymbol{\barbelow{\theta}}_{t^{\prime}}^*, \boldsymbol{\bar \theta_{t^{\prime}}^*}]\leftarrow$ 
                \State
                $\leftarrow \hypnet([ \frac{\gamma}{M}\cos\left(\preem_{t^{\prime}}\right) - \epsilon^*_{t^{\prime}},  \frac{\gamma}{M}\cos\left(\preem_{t^{\prime}}\right) +\epsilon^*_{t^{\prime}}]; \hypwei^*)$
                 
            \EndFor
        \EndIf
        \For{$i \leftarrow 1$ to $n$}
            \If{$i \leq \floor*{\frac{n}{2}}$}
                \State
                $\hat{\epsilon}^{*}_i \leftarrow \frac{i}{\floor*{\frac{n}{2}}} \cdot \gamma$
            \Else
                \State
                $\hat{\epsilon}^{*}_i \leftarrow \gamma$
            \EndIf
            \State
            $\epsilon^{*}_t \leftarrow \hat{\epsilon}^{*}_i\cdot\sigma\left( \epsilon^{*}_t \right)$ 
            \State 
            $[\boldsymbol{\barbelow \theta_t}, \boldsymbol{\bar \theta_t}] \leftarrow$
            \State
            $\leftarrow \hypnet([ \frac{\gamma}{M}\cos\left(\preem_t\right) - \epsilon^{*}_t,  \frac{\gamma}{M}\cos\left(\preem_t\right) +\epsilon^{*}_t]; \hypwei)$
            \State 
            $[\barbelow{\hat{y}}_{i, t}, \bar{\hat{y}}_{i, t}] \leftarrow \phi(\boldsymbol{x}_{i, t};  [\boldsymbol{\barbelow \theta_t}, \boldsymbol{\bar \theta_t}])$
            \If{$t = 1$}  
                \State 
                $\mathcal{L} \leftarrow \mathcal{L}_{current}$\;
            \Else
                \State 
                $\mathcal{L} \leftarrow \mathcal{L}_{current} + \beta \cdot \mathcal{L}_{output}$
            \EndIf
            \State 
            Update $\hypwei$, $a_t$ and $\epsilon_t$
        \EndFor
        \State
        Freeze and store $\boldsymbol{a}_t$ and $\epsilon_t$
    \EndFor
    \end{algorithmic}
    \end{algorithm}

\section{Lipschitz regularity of an interval-based hypernetwork}\label{sec:lipschitz_proof}
    \begin{prop}\label{prop:lipschitz_property}
    Let $(e_1, e_2, \ldots, e_T)$ be embedding centers, $T$ be the number of CL tasks, $\gamma > 0$ be a perturbation value, %$\epsilon^*$ be a ve ector of ones, 
    $\hypnet(\cdot; \hypwei)$ be an MLP-based hypernetwork with weights $\hypwei = \langle \eta_1, b_1,\ldots, \eta_L, b_L \rangle$ and an activation function, $\psi\left(\cdot\right)$, is Lipschitz continuous with a real positive constant $K$, $M$ be a natural number representing the dimensionality of the embedding space, $\barbelow e_t$, $\bar e_t$ be calculated according to Section \ref{sec:interval_hypernetwork} \eqref{preem2em} with $\epsilon_t \equiv \epsilon^*$ being a vector of ones, where $t \in \{ 1, 2, \ldots, T \}$, $L$ be a number of the hypernetwork's layers. 
    Then, for each $t$
    \begin{equation*}
        \sup_{\mu,\xi\in\left[ \munderbar{e}_t, \bar{e}_t \right]}\norm{}{\hypnet\left( \mu;\eta \right) - \hypnet\left( \xi;\eta \right)}_2 \leq K^L \prod_{i=1}^{L}\norm{}{\eta_i}_2 \norm{}{\munderbar{e}_t - \bar{e}_t}_2.
    \end{equation*}       
    % \PawelW{argumentem $\psi$ nie jest embedding tylko liczba rzeczywista. Pytanie z resztą, czy $\psi'$ tu traktujemy jako wektor pochodnych funkcji aktywancj dla całej warstwy, czy pojedynczego neuronu? Bo jak pojedynczego neuronu, to norma $\| \|$ nie ma sensu. A jak warstwy, to na norma może być sporo większa niż 1}
    \begin{proof}[Proof of Proposition \ref{prop:lipschitz_property}]
        From Proposition 1, \cite{scaman2019lipschitzregularitydeepneural}, and the fact that the hypernetwork is an MLP-based neural network with a Lipschitz continuous activation function, for any $\mu,\xi\in \left[ \barbelow{e}_t, \bar{e}_t \right]$, $t\in \{ 1,\ldots, T \}$, we have:
        \begin{equation*}
            \norm{}{\hypnet\left( \mu;\eta \right) - \hypnet\left( \xi;\eta \right)}_2 \leq K^L\prod_{i=1}^{L}\norm{}{\eta_i}_2\cdot \norm{}{\mu - \xi}_2.
        \end{equation*}
        It now suffices to take the supremum over $\mu,\xi \in \left[ \barbelow{e}_t, \bar{e}_t \right]$ on both sides of the above inequality:
        \begin{equation*}
        \sup_{\mu,\xi\in\left[ \munderbar{e}_t, \bar{e}_t \right]}\norm{}{\hypnet\left( \mu;\eta \right) - \hypnet\left( \xi;\eta \right)}_2 \leq K^L\prod_{i=1}^{L}\norm{}{\eta_i}_2 \norm{}{\barbelow{e}_t - \bar{e}_t}_2,
        \end{equation*}
        where
        \begin{equation*}
            \sup_{\mu,\xi\in\left[ \munderbar{e}_t, \bar{e}_t \right]}\norm{}{\mu - \xi}_2 = \norm{}{\barbelow{e}_t - \bar{e}_t}_2.
        \end{equation*}
    \end{proof}
    \end{prop}
A similar result can be proven for convolutional neural networks. For more details, please see \cite{scaman2019lipschitzregularitydeepneural}.
\section{Proof of non-forgetting in \our{}}
\label{sec:proof}

    \begin{proof}[Proof of Lemma \ref{lemma:intersection}]
    Take arbitrary $t_1, t_2 \in\lbrace 1, 2, \ldots, T\rbrace$ and $i$-th coordinate of the embeddings $e_{t_1}, e_{t_2}$, where $i\in\lbrace 1, 2, \ldots, M\rbrace$. Suppose the worst-case scenario, where
    \begin{align*}
        \cos\left(e_{t_1}^{(i)}\right) &= 1, \\
        \cos\left(e_{t_2}^{(i)}\right) &= -1.
    \end{align*}
    Then the absolute difference between $\frac{\gamma}{M}\cos\left(e_{t_1}^{(i)}\right)$ and $\frac{\gamma}{M}\cos\left(e_{t_2}^{(i)}\right)$ will be the biggest and can be calculated as
    \begin{align*}
        \left| \frac{\gamma}{M}\cos\left(e_{t_1}^{(i)}\right) - \frac{\gamma}{M}\cos\left(e_{t_2}^{(i)}\right) \right| &= \frac{2\gamma}{M} \\
        &= 2\gamma\cdot \sigma\left( \epsilon^{*} \right)^{(i)}.
    \end{align*}
    Note that the maximum distance at the $i$-th coordinate between transformed centers of the intervals is equal to $\frac{2\gamma}{M}$, but this value is the same as $2\gamma \cdot \sigma\left( \epsilon^{*} \right)^{(i)}$. This means that in the worst case the intersection is a single point. As $t_1$, $t_2$ and $i$ were chosen arbitrary, the intersection
    $$
    [\barbelow{e}, \bar{e}] = \bigcap_{t=1}^T [\barbelow{e}_t, \bar{e}_t]
    $$
    is non-empty, what ends the proof.
    \end{proof}
\begin{proof}[Proof of Theorem \ref{theorem:non_forgetting}]
    As we know that the intersection $\left[ \barbelow{e}, \bar{e} \right]$ is non-empty, we just need to show that $A_t \subset A$ for every $t\in\lbrace 1, \ldots, T\rbrace$. Please note that $\lbrace A_t \rbrace_{t=1}^T$ is a sequence of non-empty, compact subsets of $R^N$. Every compact set is bounded, which implies that
    \begin{equation}
    \sup\left(\bigcap_{t=1}^T A_t \right) = \min_{t\in\lbrace 1,\ldots, T\rbrace}\sup\left( A_t \right).
    \end{equation}
    We can use this identity to prove that $A \subset A_{t^*}$, where $t^{*}\in\lbrace 1, \ldots, T\rbrace$ is arbitrary chosen. It is straightforward to see that
    {\small
    \begin{equation}
    \begin{aligned}
A 
&= \left\{ \mu \,\middle|\, \forall \epsilon > 0,\ \forall x \in X_t,\ \exists y \in Y_t:\ \sup_{\mu \in \bigcap_{t=1}^T [\underline{e}_t, \bar{e}_t]} C(\mu) \leq \epsilon \right\} \\
&= \left\{ \mu \,\middle|\, \forall \epsilon > 0,\ \forall x \in X_t,\ \exists y \in Y_t:\ \min_{t\in\{ 1,\ldots,T \}} \sup_{\mu \in [\underline{e}_t, \bar{e}_t]} C(\mu) \leq \epsilon \right\} \\
&\subseteq \left\{ \mu \,\middle|\, \forall \epsilon > 0,\ \forall x \in X_t,\ \exists y \in Y_t:\ \forall_{t\in\{ 1,\ldots, T \}} \sup_{\mu \in [\underline{e}_t, \bar{e}_t]} C(\mu) \leq \epsilon \right\} \\
&= \bigcap_{t=1}^T \left\{ \mu \,\middle|\, \forall \epsilon > 0,\ \forall x \in X_t,\ \exists y \in Y_t:\ \sup_{\mu \in [\underline{e}_t, \bar{e}_t]} C(\mu) \leq \epsilon \right\} \\
&= \bigcap_{t=1}^T A_t \subseteq A_{t^*}.
\end{aligned}
    \end{equation}
    where 
    \begin{equation}
    C\left(\mu\right) = \norm{}{y - \phi\left(x, \hypnet\left(\mu; \hypwei_T\right)\right)}_2.
    \end{equation}
    }
    As $t^*$ was arbitrary chosen, we infer that for each $t\in\lbrace 1, \ldots, T\rbrace$ $A \subset A_t$, which ends the proof.
\end{proof}

\section{Relaxation technique}\label{sec:relax_tech}
Based on our experiments and those conducted in \cite{wolczyk2022continual}, we have observed that training convolutional neural networks in an interval-based setup presents significant challenges. We argue that many of these difficulties stem from the absence of a straightforward substitute for an interval-based batch normalization layer. To address these issues, we propose mitigating the challenges by relaxing the interval constraints in the target network while maintaining a fully interval-based hypernetwork. The pseudocode for this approach is provided in Algorithm \ref{alg:relax_tech}.

While the relaxation technique facilitates effective training of convolutional neural networks, it also has a certain drawback, because the image of the target network obtained through the relaxation technique is only contained within the image of the target network trained using the fully interval-based method. This outcome is illustrated in Proposition \ref{prop:images_inclusion}.
\begin{prop}\label{prop:images_inclusion}
    Let $(e_1, e_2, \ldots, e_T)$ be embedding centers, $T$ be the number of CL tasks, $\gamma > 0$ be a perturbation value, %$\epsilon^*$ be a ve ector of ones, 
    $\hypnet(\cdot; \hypwei)$ be a hypernetwork with weights $\hypwei$, $M$ be a natural number representing the dimensionality of the embedding space, $\barbelow e_t$, $\bar e_t$ be calculated according to \ref{sec:interval_hypernetwork} \eqref{preem2em} with $\epsilon_t \equiv \epsilon^*$ being a vector of ones, where $t \in \{ 1, 2, \ldots, T \}$. Let also $\phi\left( \cdot;\left[ \barbelow{\theta}_t, \bar{\theta}_t \right] \right)$ be a target network with a non-negative and non-decreasing activation function $\psi\left(\cdot\right)$ and with interval weights generated by the hypernetwork. Then, for each $t\in\lbrace 1,\ldots,T \rbrace$ and non-negative observation $x$:
    \begin{equation*}
        \left[ \phi_{\text{min}}\left(x\right), \phi_{\text{max}}\left(x\right) \right] \subseteq \phi\left( x;\left[ \barbelow{\theta}_t, \bar{\theta}_t \right] \right),
    \end{equation*}
    where
    \begin{align*}
    \phi_{\text{min}}\left(x\right) &= \min\lbrace \phi\left( x; \barbelow{\theta}_t \right), \phi\left( x; \bar{\theta}_t \right) \rbrace, \\
        \phi_{\text{max}}\left(x\right) &= \max\lbrace \phi\left( x; \barbelow{\theta}_t \right), \phi\left( x; \bar{\theta}_t \right) \rbrace.
    \end{align*}
    \begin{proof}
        It is straightforward to see that:
    \begin{align*}
    \phi_{\text{min}}\left(x\right) &\geq \min_{\theta\in\left[ \munderbar{\theta}_t, \bar{\theta}_t \right]} \phi\left( x; \theta\right), \\
    \phi_{\text{max}}\left(x\right) &\leq \max_{\theta\in\left[ \munderbar{\theta}_t, \bar{\theta}_t \right]} \phi\left( x; \theta\right).
    \end{align*}
    This ends the proof, because
    \begin{equation*}
        \left[\min_{\theta\in\left[ \munderbar{\theta}_t, \bar{\theta}_t \right]} \phi\left( x; \theta\right), \max_{\theta\in\left[ \munderbar{\theta}_t, \bar{\theta}_t \right]} \phi\left( x; \theta\right)\right] = \phi\left( x;\left[ \barbelow{\theta}_t, \bar{\theta}_t \right] \right).
    \end{equation*}
    \end{proof}
\end{prop}

One might ask whether equality can occur in Proposition \ref{prop:images_inclusion}. The answer is yes, and this is demonstrated in Proposition \ref{prop:images_equality}.

\begin{prop}\label{prop:images_equality}
    Let us assume that all the assumptions of Proposition \ref{prop:images_inclusion} are satisfied and that the notation is preserved. Additionally, assume that the interval weights $[\barbelow{\theta}_t, \bar{\theta}_t] = \langle [\barbelow{W}_1^{\left(t\right)}, \bar{W}_1^{\left(t\right)}], [\barbelow{b}_1^{\left(t\right)}, \bar{b}_1^{\left(t\right)}], \ldots, [\barbelow{W}_L^{\left(t\right)}, \bar{W}_L^{\left(t\right)}], [\barbelow{b}_L^{\left(t\right)}, \bar{b}_L^{\left(t\right)}]\rangle$ produced by the hypernetwork are non-negative, where $L$ denotes the number of layers in the target network. Then, for each $t\in\lbrace 1,\ldots,T \rbrace$ and non-negative observation $x$:
    \begin{equation*}
        \left[ \phi_{\text{min}}\left(x\right), \phi_{\text{max}}\left(x\right) \right] = \phi\left( x;\left[ \barbelow{\theta}_t, \bar{\theta}_t \right] \right).
    \end{equation*}
    \begin{proof}
        Fix arbitrary $t\in\lbrace 1, \ldots, T\rbrace$ and let us introduce the following notation for any $k\in\lbrace 1,\ldots, L-1 \rbrace$:
        \begin{align*}
            \bar{z}_{k}^{\left(t\right)} &= \psi\left( \bar{W}_{k}^{\left(t\right)}\cdot\bar{z}_{k-1}^{\left(t\right)} + \bar{b}_{k}^{\left(t\right)} \right), \\
            \barbelow{z}_{k}^{\left(t\right)} &= \psi\left( \barbelow{W}_{k}^{\left(t\right)}\cdot\barbelow{z}_{k-1}^{\left(t\right)} + \barbelow{b}_{k}^{\left(t\right)} \right), \\
            \barbelow{z}_0^{\left(t\right)} &= x, \\
            \bar{z}_0^{\left(t\right)} &= x.
        \end{align*}
        Based on the assumptions, we know that the weights generated by the hypernetwork are non-negative and that $\psi\left(\cdot\right)$ is a non-negative valued and non-decreasing activation function. This implies that for any $\barbelow{\alpha}_k^{\left(t\right)},\bar{\alpha}_k^{\left(t\right)}\in\left[\barbelow{W}_k^{\left(t\right)}, \bar{W}_k^{\left(t\right)}\right]$, $\barbelow{\beta}_k^{\left(t\right)},\bar{\beta}_k^{\left(t\right)}\in\left[\barbelow{b}_k^{\left(t\right)}, \bar{b}_k^{\left(t\right)}\right]$ such that $\barbelow{\alpha}_k^{\left(t\right)} \leq \bar{\alpha}_k^{\left(t\right)}$, $\barbelow{\beta}_k^{\left(t\right)} \leq \bar{\beta}_k^{\left(t\right)}$, $k \in {1, \ldots, L}$:
        \begin{align*}
            \barbelow{\alpha}_k^{\left(t\right)}\cdot\barbelow{z}_{k-1}^{\left(t\right)} + \barbelow{\beta}_k^{\left(t\right)} &\leq \bar{\alpha}_k^{\left(t\right)}\cdot\bar{z}_{k-1}^{\left(t\right)} + \bar{\beta}_k^{\left(t\right)},  \\
            \psi\left(\barbelow{\alpha}_k^{\left(t\right)}\cdot\barbelow{z}_{k-1}^{\left(t\right)} + \barbelow{\beta}_k^{\left(t\right)}\right) &\leq \psi\left(\bar{\alpha}_k^{\left(t\right)}\cdot\bar{z}_{k-1}^{\left(t\right)} + \bar{\beta}_k^{\left(t\right)}\right).
        \end{align*}
        Hence,
        \begin{align*}
            \barbelow{W}_L^{\left(t\right)}\cdot\barbelow{z}_{L-1}^{\left(t\right)} + \barbelow{b}_L^{\left(t\right)} &= \phi_{\text{min}}\left(x\right), \\
             \barbelow{W}_k^{\left(t\right)}\cdot\barbelow{z}_{k-1}^{\left(t\right)} + \barbelow{b}_k^{\left(t\right)} &= \min_{\substack{a\in\lbrace \munderbar{W}_{k}^{\left(t\right)}, \bar{W}_{k}^{\left(t\right)}\rbrace \\ b\in\lbrace \munderbar{z}_{k-1}^{\left(t\right)}, \bar{z}_{k-1}^{\left(t\right)}\rbrace}} \lbrace a \cdot b +\barbelow{b}_k^{\left(t\right)}\rbrace.
        \end{align*}
        Analogously,
        \begin{align*}
        \bar{W}_L^{\left(t\right)}\cdot\bar{z}_{L-1}^{\left(t\right)} + \bar{b}_L^{\left(t\right)} &= \phi_{\text{max}}\left(x\right), \\
             \bar{W}_k^{\left(t\right)}\cdot\bar{z}_{k-1}^{\left(t\right)} + \bar{b}_k^{\left(t\right)} &= \max_{\substack{a\in\lbrace \munderbar{W}_{k}^{\left(t\right)}, \bar{W}_{k}^{\left(t\right)}\rbrace \\ b\in\lbrace \munderbar{z}_{k-1}^{\left(t\right)}, \bar{z}_{k-1}^{\left(t\right)}\rbrace}} \lbrace a \cdot b +\bar{b}_k^{\left(t\right)}\rbrace.
        \end{align*}
        Then,
        \begin{align*}
            \left[ \phi_{\text{min}}\left(x\right), \phi_{\text{max}}\left(x\right) \right] &= \left[\barbelow{W}_L^{\left(t\right)}, \bar{W}_L^{\left(t\right)} \right] \cdot \left[\barbelow{z}_{L-1}^{\left(t\right)}, \bar{z}_{L-1}^{\left(t\right)} \right] \\ &+ \left[\barbelow{b}_L^{\left(t\right)}, \bar{b}_L^{\left(t\right)} \right] \\ &= \phi\left(x,\left[ \barbelow{\theta}_t, \bar{\theta}_t \right]\right),
        \end{align*}
        % \PawelW{Trzeba to zmieścić w szpalcie}
        what ends the proof.
    \end{proof}
\end{prop}

\begin{algorithm}[!t]
    \caption{The pseudocode of \our{} relaxation technique.}
    \label{alg:relax_tech}
    \begin{algorithmic}
    \Require 
    %{\bf Require:} 
    hypernetwork $\hypnet$ with weights $\hypwei$, target network $\phi$, softmax function $\sigma(\cdot)$, perturbation value $\gamma > 0$, regularization strength $\beta > 0$, cross-entropy components strength
    % parameter controlling the convex combination of standard cross-entropy and worst-case cross-entropy
    $\kappa \in \left(0, 1\right)$, dimensionality of the embedding space $M$, $n$ training iterations, datasets $\lbrace D_1, D_2, ..., D_T\rbrace$, 
    $(x_{i,t}, y_{i,t}) \in D_{t}, t \in \lbrace 1, \ldots, T \rbrace$, $i\in \lbrace 1, \ldots, N_t \rbrace$, $N_t$ is a number of samples in the dataset $D_t$.

    \Ensure 
    %{\bf Ensure:} 
    updated hypernetwork $\hypnet$ weights $\hypwei$
    
    %\BlankLine
    \State 
    Initialize randomly weights $\hypwei$ with pre-embeddings $(\preem_1, \preem_2, ..., \preem_T)$ and corresponding perturbation vectors $(\epsilon^{*}_1, \epsilon^{*}_2, \ldots, \epsilon^{*}_T)$

    \For{$t \leftarrow 1$ to $T$} 
        \If{$t > 1$}  
            \State 
            $\hypwei^{*} \leftarrow  \hypwei$

            \For{$t^{\prime} \leftarrow 1$ to $t-1$}
                \State
                $\epsilon^{*}_{t^{\prime}} \leftarrow \gamma\cdot\sigma\left( \epsilon^{*}_{t^{\prime}} \right)$
                \State
                Store $[\boldsymbol{\barbelow{\theta}}_{t^{\prime}}^*, \boldsymbol{\bar \theta_{t^{\prime}}^*}]\leftarrow$ 
                \State
                $\leftarrow \hypnet([ \frac{\gamma}{M}\cos\left(\preem_{t^{\prime}}\right) - \epsilon^*_{t^{\prime}},  \frac{\gamma}{M}\cos\left(\preem_{t^{\prime}}\right) +\epsilon^*_{t^{\prime}}]; \hypwei^*)$
                 
            \EndFor
        \EndIf
        \For{$i \leftarrow 1$ to $n$}
            \If{$i \leq \floor*{\frac{n}{2}}$}
                \State
                $\hat{\epsilon}^{*}_i \leftarrow \frac{i}{\floor*{\frac{n}{2}}} \cdot \gamma$
            \Else
                \State
                $\hat{\epsilon}^{*}_i \leftarrow \gamma$
            \EndIf
            \State
            $\epsilon^{*}_t \leftarrow \hat{\epsilon}^{*}_i\cdot\sigma\left( \epsilon^{*}_t \right)$ 
            \State 
            $[\boldsymbol{\barbelow \theta_t}, \boldsymbol{\bar \theta_t}] \leftarrow$
            \State
            $\leftarrow \hypnet([ \frac{\gamma}{M}\cos\left(\preem_t\right) - \epsilon^{*}_t,  \frac{\gamma}{M}\cos\left(\preem_t\right) +\epsilon^{*}_t]; \hypwei)$
            \State 
            $\barbelow{\hat{y}}_{i, t} \leftarrow \phi(\boldsymbol{x}_{i, t};\boldsymbol{\barbelow \theta_t})$
            \State
            $\bar{\hat{y}}_{i, t} \leftarrow \phi(\boldsymbol{x}_{i, t};\boldsymbol{\bar \theta_t})$
            \State
            $\left[\barbelow{\hat{y}}_{i, t}, \bar{\hat{y}}_{i, t}\right] \leftarrow \left[ \min\lbrace \barbelow{\hat{y}}_{i, t}, \bar{\hat{y}}_{i, t} \rbrace,  \max\lbrace \barbelow{\hat{y}}_{i, t}, \bar{\hat{y}}_{i, t} \rbrace\right]$
            \If{$t = 1$}  
                \State 
                $\mathcal{L} \leftarrow \mathcal{L}_{current}$\;
            \Else
                \State 
                $\mathcal{L} \leftarrow \mathcal{L}_{current} + \beta \cdot \mathcal{L}_{output}$
            \EndIf
            \State 
            Update $\hypwei$, $a_t$ and $\epsilon_t$
        \EndFor
        \State
        Freeze and store $\boldsymbol{a}_t$ and $\epsilon_t$
    \EndFor
    \end{algorithmic}
    \end{algorithm}

\section{Training details}
\label{sec:train-details}

\paragraph{Datasets and CL setup} We use the following datasets: 1) Permuted MNIST-10 \cite{oshg2019hypercl}, consisting of 28x28 pixel grey-scale images of 10 classes of digits, where each
task is obtained by applying a random permutation to the input image pixels, with a typical length of T~=~10 tasks;
2) Split MNIST \cite{oshg2019hypercl}, containing tasks designed by sequentially pairing the digits to introduce task overlap, forming T~=~5 binary classification tasks;
3) Split CIFAR-100 \cite{goswami2024fecam}, consisting of 32×32 pixel color images of 100 classes;
%, with 500 and 100 examples per class for training and testing, respectively;
4) Split CIFAR-10 \cite{wolczyk2022continual}, with T~=~5 binary classification tasks;
5) TinyImageNet \cite{goswami2024fecam}, a subset of ImageNet \cite{russakovsky2015imagenet}, consisting of 64×64 pixel color images of 200 classes;
%, with 500 and 50 examples per class for training and testing, respectively
6) Permuted MNIST-100 \cite{goswami2024fecam}, similar to the Permuted MNIST-10 dataset, but with T~=~100 tasks, 10 classes each, only used in the ablation study.
When the task identity is unknown during inference, for Split CIFAR-100 we experiment with a setup of 5 tasks with 20 classes each, we do not conduct experiments on TinyImageNet in this setup.
When the task identity is known, we pick a setup with 10 tasks with 10 classes each and 20 tasks with 5 classes each for Split CIFAR-100 and 40 tasks with 5 classes each for TinyImageNet.

\paragraph{Training setup}
We apply three typical CL training setups: TIL, DIL, and CIL. 
When considering \our{} in the TIL scenario, we do not use any nesting approaches, and the input embedding intervals may have different lengths. Similarly, in the CIL setup, we also do not use nesting approaches; however, we employ entropy to infer the task identity. In the DIL scenario, we use the $\cos(\cdot)$ embedding nesting method and the input intervals are of the same length, to prevent the embeddings from collapsing to a trivial point.

\paragraph{Relaxation technique}Depending on the complexity of the target network architecture, we have two interval learning approaches: 1)~fully interval technique, meaning that we apply intervals to both the hypernetwork and target network; 2) interval relaxation technique, meaning we use a fully interval hypernetwork and a~non-interval target network, where we only restore the order of predictions after its last layer, hence we can also use vanilla batch normalization.
Regardless of the training setup (known / unknown task identity): 1) for target networks which are MLPs, we always use the fully interval training technique;
2) in cases when a convolutional target network is used (ResNet-18 and AlexNet), we apply the interval relaxation technique to the training process, since convolutional layers are generally more difficult to train with intervals \cite{wolczyk2022continual}.

\paragraph{Form of hypernetwork regularization loss}
The interval weights produced by the hypernetwork are contained within the interval
$$
\left[\min_{\mu\in\left[ \munderbar{e}_t, \bar{e}_t \right]}\hypnet\left(\mu;\hypwei\right), \max_{\mu\in\left[ \munderbar{e}_t, \bar{e}_t \right]}\hypnet\left(\mu;\hypwei\right)\right].
$$
As previously mentioned, the interval weights produced by the hypernetwork tend to be short, and regularization of the interval endpoints is typically sufficient. However, it is worth considering what would happen if regularization were applied to points within the interval that are not endpoints. We examined three types of regularization: (1) regularization applied only to the endpoints of the interval, (2) regularization applied to both the endpoints and the midpoint of the interval, and (3) regularization applied to the endpoints and randomly selected points from within the interval. We observed that the second and third approaches yielded slightly better results than the first. Although there was no significant difference between the second and third approaches, we opted to use the second approach.

\paragraph{Common hyperparameters} In each training process, we use the ReLU activation function and apply the Adam optimizer with different learning rates but the same default betas coefficients, i.e., $\beta_1 = 0.9$, $\beta_2 = 0.999$. Also, anytime we specify that a learning rate scheduler is used, we apply the ReduceLROnPlateau scheduler from PyTorch. There is a parameter $\kappa$  which is responsible for weighing the worst case loss wrt. to the basic entropy loss, as it takes part in their convex linear combination. We always use $\kappa = 0.5$. However, this hyperparameter is scheduled during training, so instead of using the same $\kappa$~value throughout the training, we use $\kappa_i = \max\lbrace 1 - 0.00005\cdot i, \kappa \rbrace$, where $i$ is the current training iteration. In other words, it means that at the beginning of the training we put more attention to the proper classification on the centers of the intervals and then, gradually, we put more attention to the worst case component. Furthermore, there is another parameter, namely $\beta$, which is responsible for the hypernetwork regularization strength, weighing the part of loss responsible for non-forgetting.

\paragraph{Hardware and software resources used} 
We implemented \our{} using Python 3.7.6, incorporating libraries such as hypnettorch 0.0.4 (von Oswald et al., 2019), PyTorch 1.5.0 with CUDA support, NumPy 1.18, Pandas 1.0.1, Matplotlib 3.1.3, seaborn 0.10.0, among others. Most of our computations are conducted on an NVIDIA GeForce RTX 4090, but some training sessions are also performed using NVIDIA GeForce RTX 3080, NVIDIA GeForce RTX 2070, NVIDIA A100, and NVIDIA DGX graphic cards.

\paragraph{Known task identity hyperparameters}
As it was stated in the main section of this work, here we do not use any interval nesting methods. Also, in this setting, we randomly initialize each task embedding. Let us cover the list of parameters used in this setup:

\begin{itemize}

\item Split MNIST -- the hypernetwork and target network used are two-layered MLPs with 75 neurons per layer and 400 neurons per layer, respectively. We use data augmentation as in \cite{von2019continual}, the Adam optimizer with a learning rate $lr = 0.001$, no scheduler, batches of size 128, embeddings with 72 dimensions and perturbation value $\gamma = 1$. Moreover, we use $\beta = 0.01$ and we conduct the training for 2000 iterations.

\item Permuted MNIST-10 and Permuted MNIST-100 -- the hypernetwork and target network used are two-layered MLPs with 100 neurons per layer and 1000 neurons per layer, respectively. We do not use any data augmentation and we choose the Adam optimizer with a~learning rate $lr = 0.001$, no scheduler. We use batch size equal to 128, embeddings with 24 dimensions, perturbation value $\gamma = 0.5$, $\beta = 0.01$ and conduct the training for 5000 iterations.

\item Split CIFAR-10 -- we use a one-layered MLP with 100 neurons as the hypernetwork and AlexNet with batch normalization as the target network. We do not apply any data augmentation. We use the Adam optimizer with a learning rate $lr = 0.001$ and the $lr$ scheduler. We use batch size equal to 128, embeddings with 48 dimensions, perturbation value $\gamma = 0.5$, $\beta = 0.001$ and conduct the training for 2000 iterations.

\item Split CIFAR-100 (10 tasks, 10 classes each) -- we use a one-layered MLP with 100 neurons as the hypernetwork and ResNet-18 with batch normalization as the target network. We apply the data augmentation as in \cite{goswami2024fecam}, use the Adam optimizer with a learning rate $lr = 0.001$ and the $lr$ scheduler. We use batch size equal to 32, embeddings with 48 dimensions, perturbation value $\gamma = 1$, $\beta = 0.01$ and conduct the training for 200 epochs.

\item Split CIFAR-100 (20 tasks, 5 classes each) -- we use a one-layered MLP with 100 neurons as the hypernetwork and AlexNet with batch normalization as the target network. We do not apply any data augmentation, use the Adam optimizer with a learning rate $lr = 0.001$ and the $lr$ scheduler. We use batch size equal to 32, embeddings with 24 dimensions, perturbation value $\gamma = 0.1$, $\beta = 0.01$ and conduct the training for 50 epochs.

\item TinyImageNet -- we use a one-layered MLP with 200 neurons per layer as the hypernetwork and ResNet-18 with batch normalization as the target network. We apply the data augmentation, use the Adam optimizer with a learning rate $lr = 0.0001$ and the $lr$ scheduler. We use batch size equal to 256, embeddings with 700 dimensions, $\beta = 0.1$ and we conduct the training for 10 epochs.
For this dataset, we choose the perturbation value $\gamma = 0.1$.

\end{itemize}
The same hyperparameters were used in the CIL setup.

\paragraph{Unknown task identity hyperparameters -- universal embedding} In this setting, we use the $\cos(\cdot)$  interval nesting method, as it gave the best performance results. Moreover, there is an additional parameter specifying a~custom embedding initialization, which means we initialize the next task's embedding as the previously learned one. In this setup, we always use the custom embedding initialization. We would like to emphasize that it is just another way of initialization, we do not change or additionally train any of the previously learned embeddings. Let us cover the parameters used in this setting:

\begin{itemize}

\item Split MNIST -- we use the same hypernetwork and target network architectures as in the previous setup, as well as data augmentation. We apply the Adam optimizer with a learning rate $lr = 0.001$, no scheduler, batches of size 64, embeddings with 24 dimensions and perturbation value $\gamma = 15$. Moreover, we use $\beta = 0.01$ and conduct the training for 2000 iterations.

\item Permuted MNIST-10 and Permuted MNIST-100 -- we use the same hypernetwork and target network architectures as in the previous setup with no data augmentation. We use the Adam optimizer with a learning rate $lr = 0.001$, no scheduler, batches of size 128, embeddings with 24 dimensions and perturbation value $\gamma = 5$. Moreover, we use $\beta = 0.01$ and conduct the training for 5000 iterations.

\item Split CIFAR-10 -- we use the same hypernetwork and target network architectures as in the previous setup, with no data augmentation. We use the Adam optimizer with a learning rate $lr = 0.001$ and the $lr$ scheduler. We use batch size equal to 128, embeddings with 48 dimensions, perturbation value $\gamma = 5$, $\beta = 0.01$ and conduct the training for 2000 iterations.

\item Split CIFAR-100 (5 tasks, 20 classes per each) -- we use the same hypernetwork and target network architectures as in the previous setup, with no data augmentation. We use the Adam optimizer with a learning rate $lr = 0.001$ and the $lr$ scheduler. We use batch size equal to 32, embeddings with 48 dimensions, perturbation value $\gamma = 10$, $\beta = 0.01$ and conduct the training for 20 epochs. We use data augmentation.

\end{itemize}

%%%%%%%%%%%%%%%%%%%%%%%%%%%%%%%%%%%%%%%%%%%%%%%%%%%%%%%%%%%%%%%%%%%%%%%%%%%%%%%
%%%%%%%%%%%%%%%%%%%%%%%%%%%%%%%%%%%%%%%%%%%%%%%%%%%%%%%%%%%%%%%%%%%%%%%%%%%%%%%

%%%%%%%%%%%%%%%%%%%%%%%%%%%%%%%%%%%%%%%%

%%%%%%%%%%%%%%%%%%%%%%%%%%%%%%%%%%%%%%%%

%%%%%%%%%%%%%%%%%%%%%%%%%%%%%%%%%%%%%%%%
\section{Ablation study}
\label{sec:ablation}

\subsection{Different perturbation size of intervals on Permuted MNIST-10}

In this subsection, we experiment with different perturbated values to show that size of the perturbation impacts the ability of our method to find good solutions. As stated in the main section of this work, the larger the perturbation value, the larger the embedding intervals. One could argue that increasing the perturbated value is similar to giving more difficult adversary examples on the input to our network. 
\begin{figure}[htbp]
    \centering
    \begin{subfigure}[b]{0.47\textwidth}
        \centering
        \includegraphics[width=\textwidth]{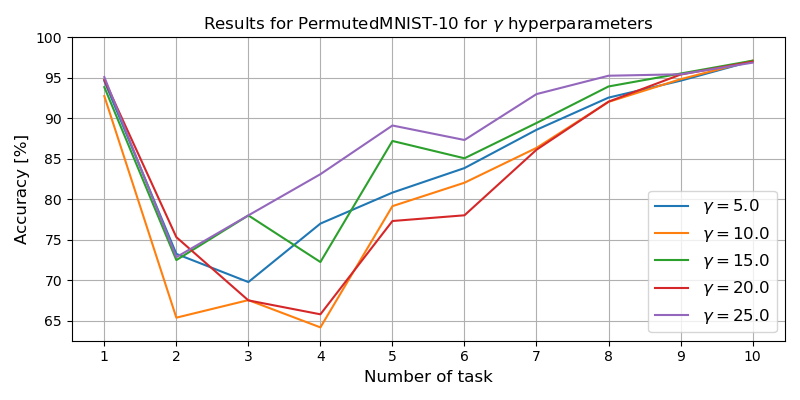}
        \caption{Results obtained using the nesting by the \newline $\cos\left(\cdot\right)$ method.}
    \end{subfigure}
    \hspace{-0.17cm}
    \begin{subfigure}[b]{0.47\textwidth}
        \centering
        \includegraphics[width=\textwidth]{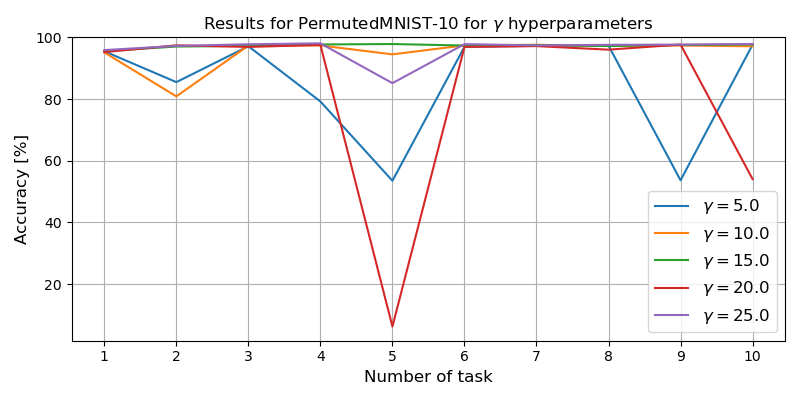}
        \caption{Results obtained for the known task identity setup. \newline}
    \end{subfigure}
    \caption{Mean test accuracy for consecutive CL tasks averaged over 2 runs of different interval size settings of \our{} for 10 tasks of Permuted MNIST-10 dataset. %\PawelW{Trzeba to zmieścić w szpalcie. Najlepiej wciskając legendy gdzieś w wykresy.}
    }
    \label{fig:gamma_grid}
\end{figure}
As previously mentioned, larger perturbated value on the input embedding intervals implicates an increase in the space where a universal embedding could be found. Hence, it could make the process of finding an intersection of intervals easier.

We consider TIL, DIL, and nesting with the $\cos(\cdot)$ method when applicable. The rest of the parameters are the same as for the best performing models. The grid search outcome is shown in Figure \ref{fig:gamma_grid}. Considering results averaged over 2 runs, the best parameter for the unknown task identity is $\gamma = 25$, while for the known task identity is $\gamma=15$. These parameters differ from our final choice in Appendix \ref{sec:train-details}, since another $\gamma$ values give better accuracy, based on the averaged results of 5 different runs.

\subsection{Interval lengths of target network weights}\label{appendix:interval_lengths_plots}

Using the interval arithmetic within a hypernetwork gives us a possibility of generating interval target network weights. Nevertheless, very often these solutions can become degraded to a point estimate in the network weight space. Since such possibility exists, it is important to check if our hypernetwork generates such output.

\begin{figure}[htbp]
    \centering
    \begin{subfigure}[b]{0.45\textwidth}
        \centering
        \includegraphics[width=\textwidth]{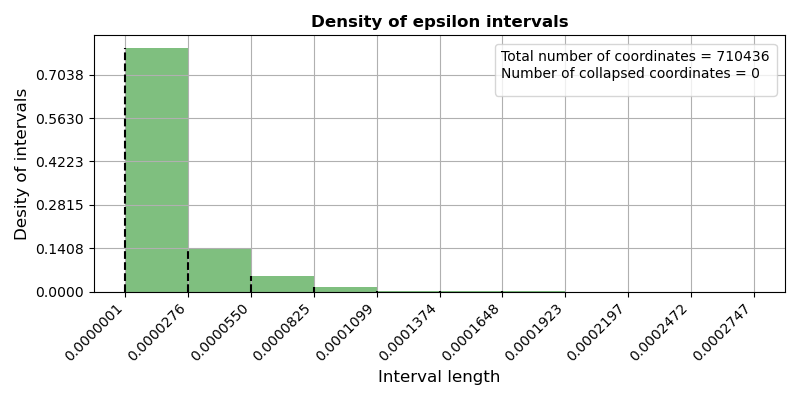}
        \caption{Results obtained for the Split CIFAR-100 dataset using the $\cos\left(\cdot\right)$ nesting method.}
    \end{subfigure}
    % \hfill
    \hspace{0.1cm}
    \begin{subfigure}[b]{0.45\textwidth}
        \centering
        \includegraphics[width=\textwidth]{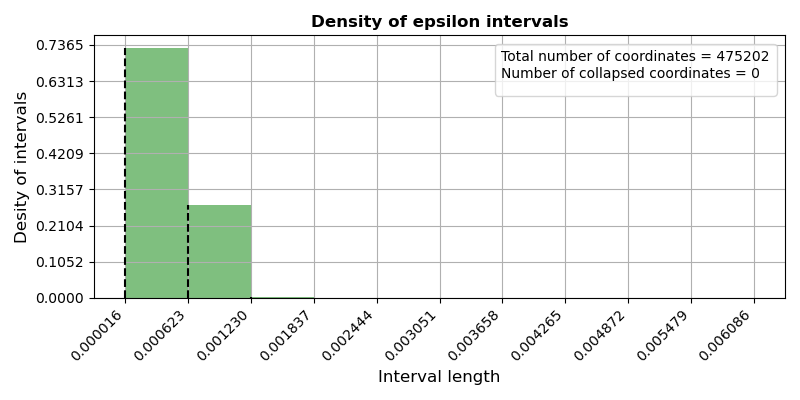}
        \caption{Results obtained for the Split MNIST using the $\cos\left(\cdot\right)$ nesting method.}
    \end{subfigure}
    \caption{Histograms are calculated for the Split MNIST and Split CIFAR-100 datasets using the MLP and ResNet-18 architectures, respectively.}
    \label{fig:histograms}
\end{figure}

In Figure \ref{fig:histograms} we present histogram plots of lenghts of target network interval weights to analyse whether this situation occurs. We also give a specific number of coordinates (weights) which collapse to a point for the sake of clarity. We show one result per each type of the target network: an MLP on the Split MNIST dataset and a convolutional network on the Split CIFAR-100 dataset, both with the $\cos(\cdot)$ nesting method.
From Figure \ref{fig:histograms} we may conclude that, for Split CIFAR-100 and Split MNIST, none of the coordinates collapse to a trivial point in the weight space, meaning we obtain non-trivial universal weights.

\subsection{Intervals around embeddings -- Permuted MNIST-10 and Split MNIST}
In this subsection, we present the acquired embedding intervals on the Permuted MNIST-10 and Split MNIST datasets. We use the DIL setting, where the task identity is unknown and the interval vector is non-learnable (fixed). Results are shown in Figure \ref{fig:intervals_embeddings_ablation} for nesting with the $\cos(\cdot)$~method.

\begin{figure}[htbp]
    \centering
    \begin{subfigure}[b]{0.45\textwidth}
        \centering
        \includegraphics[width=\textwidth]{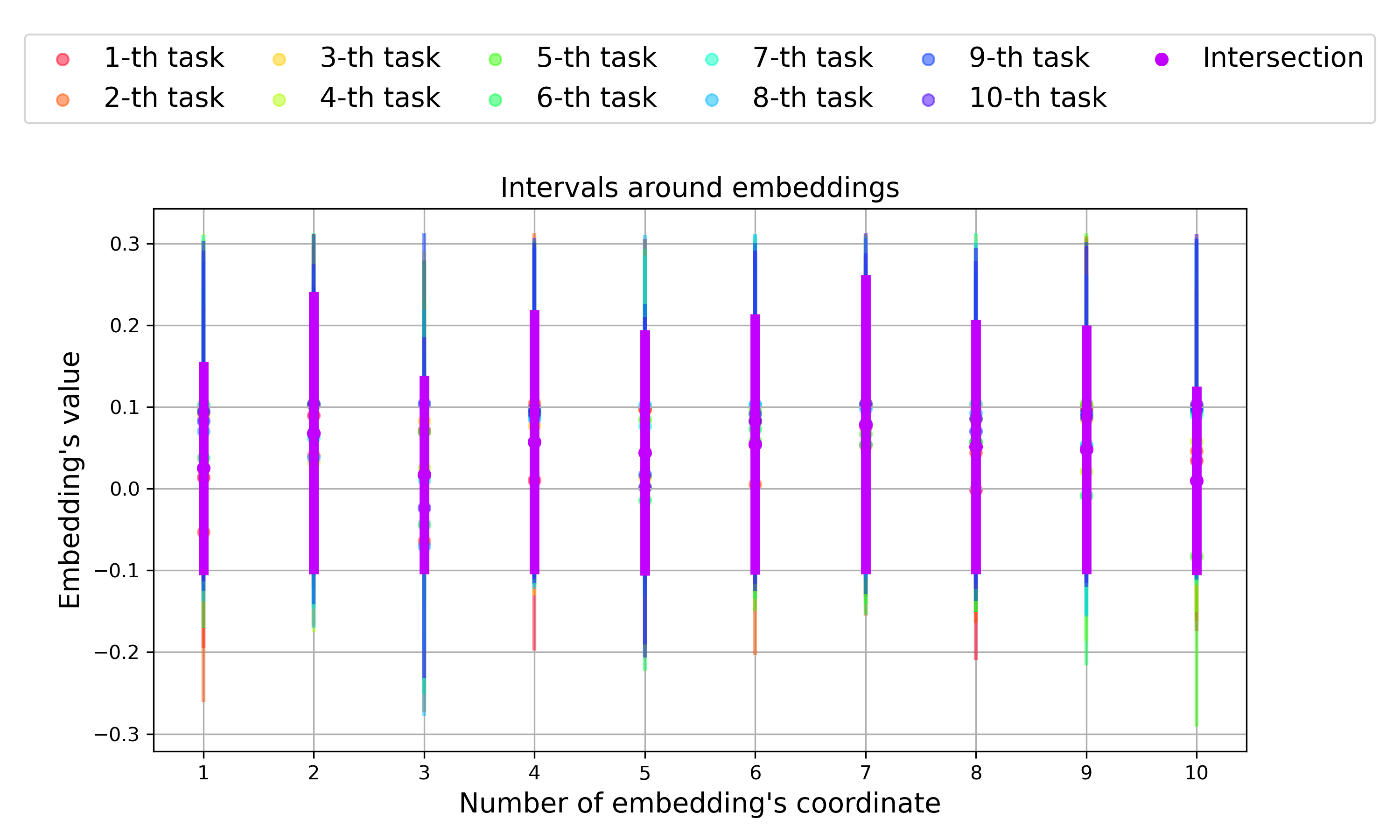}
        \caption{Results obtained for the Permuted MNIST-10 dataset using the $\cos\left(\cdot\right)$ nesting method.}
    \end{subfigure}
    \hspace{0.1cm}
    \begin{subfigure}[b]{0.45\textwidth}
        \centering
        \includegraphics[width=\textwidth]{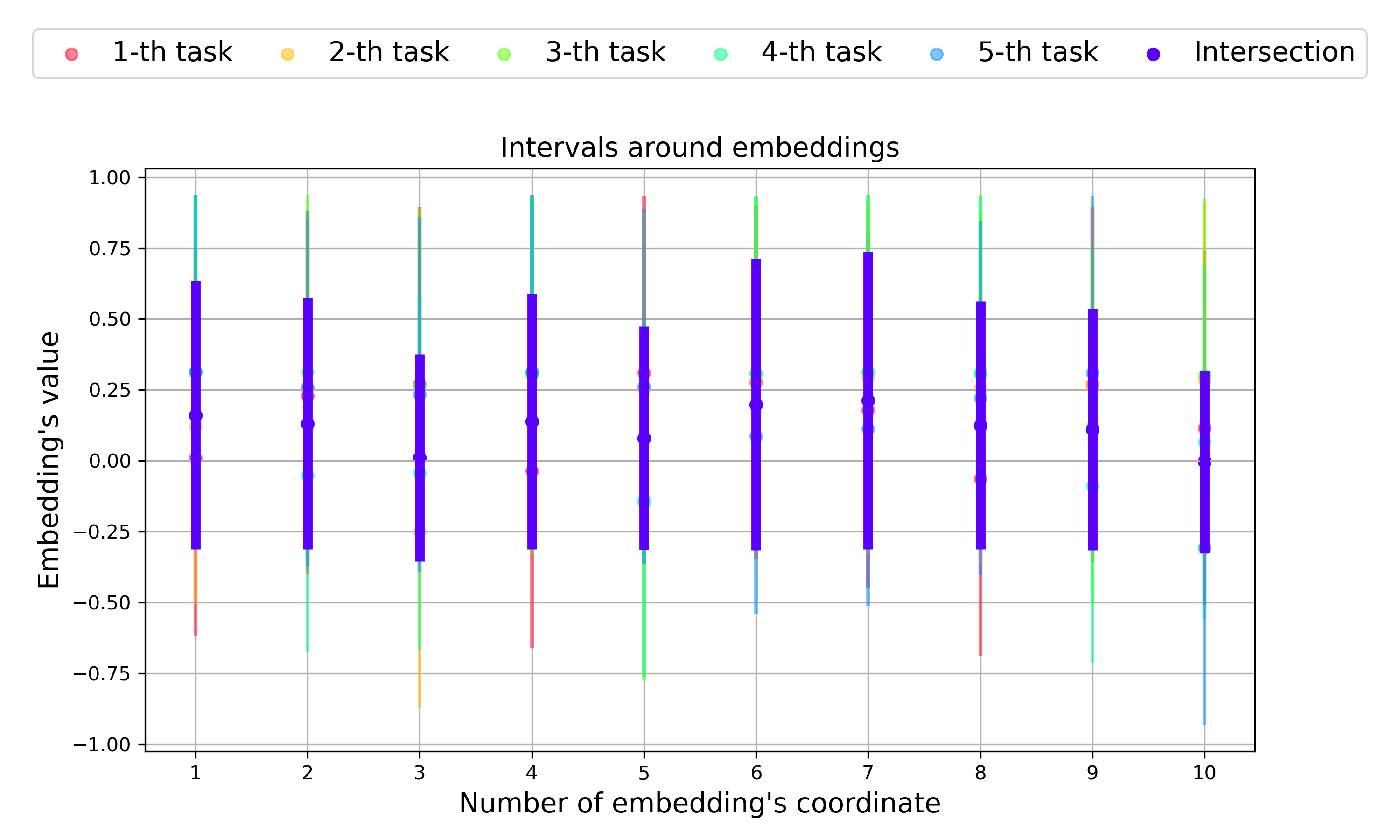}
        \caption{Results obtained for the Split MNIST dataset using the $\cos\left(\cdot\right)$ nesting method.}
    \end{subfigure}
    \caption{Ten first intervals around task embeddings for Permuted MNIST-10 and Split MNIST.}
\label{fig:intervals_embeddings_ablation}
\end{figure}

\subsection{Larger number of tasks -- Permuted MNIST-100}
It is a fair question to ask, how our method performs when there is a larger number of tasks, for example 100 instead of 10. This is a more difficult setting, since the capacity of our model to learn new tasks and remember the previous ones is limited. In such situations, catastrophic forgetting occurs much easier, due to the growing size of knowledge that could be forgotten.

\begin{figure}
    \centering
    \includegraphics[width=0.5\textwidth]{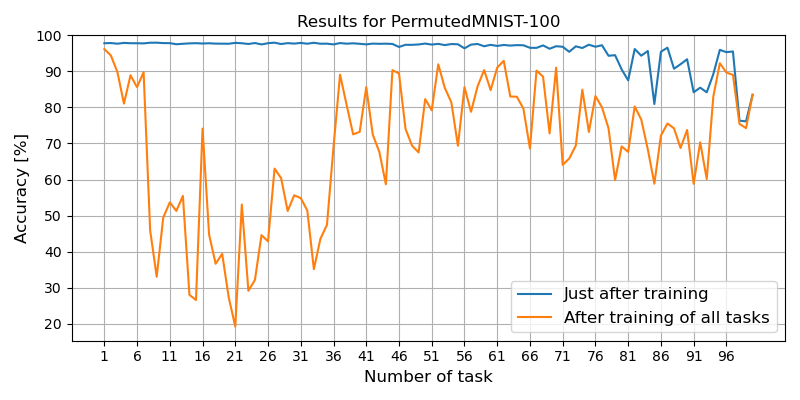}
    \caption{Mean test accuracy for consecutive CL tasks averaged over 3 runs for 100 tasks of Permuted MNIST–100 dataset.}
    \label{fig:larger_num_tasks}
\end{figure}

To answer this question, in this section, we provide an analysis of interval arithmetic learning method performance given the Permuted MNIST-100 dataset. As a reminder, this dataset consists of 100 tasks comparing pairs of images of numbers subjected to different distortions. Due to the number of tasks and the size of the dataset, we limit ourselves to provide results for the TIL setting, with learnable interval parameters and no nesting.

The results are presented in Figure \ref{fig:larger_num_tasks}. 
We observe that training on consecutive tasks has a significant impact on the accuracy on previously learned tasks. Nevertheless, our model is still able to retain some of the previous knowledge, not dropping to a random accuracy on any of the tasks. Furthermore, we conclude that training consecutive tasks in the latter stages is more difficult, due to the decrease in accuracy just after training.

\subsection{Ablation of the regularization method}

This subsection involves experimenting with different values for the $\beta$ parameter, which directly impacts the memorization factor in the loss function. This part of mathematical expression targets specifically the hypernetwork output. The experiment is done in two setups: 1) TIL, 2) DIL. We conduct a grid search on the Permuted MNIST-10 dataset and check the $\beta$ parameter from a set of values: $\{1.0, 0.1, 0.05, 0.01, 0.001\}$, while the rest of parameters are taken from the best performing models. The grid search results are presented in Figure \ref{fig:beta_grid}.

\begin{figure}[htbp]
    \centering
    \begin{subfigure}[b]{0.47\textwidth}
        \centering
        \includegraphics[width=\textwidth]{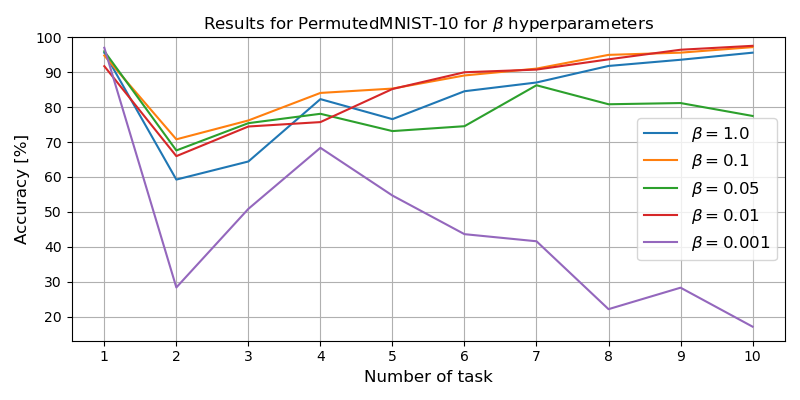}
        \caption{Results obtained in the DIL setting, using the nesting by the $\cos\left(\cdot\right)$ method.}
    \end{subfigure}
    \hspace{-0.17cm}
    \begin{subfigure}[b]{0.47\textwidth}
        \centering
        \includegraphics[width=\textwidth]{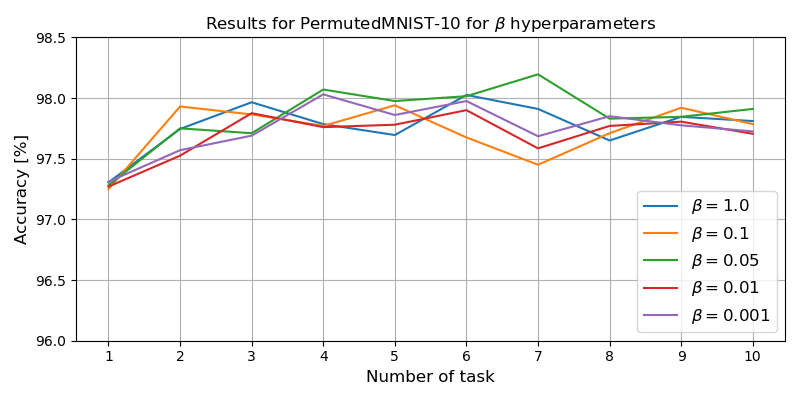}
        \caption{Results obtained in the TIL setting. \newline}
    \end{subfigure}
    \caption{Mean test accuracy for consecutive CL tasks averaged over 2 runs with different $\beta$ hyperparameters of \our{} for 10 tasks of Permuted MNIST-10 dataset. 
    % \PawelW{Trzeba to zmieścić w szpalcie. Najlepiej wciskając legendy gdzieś w wykresy.}
    }
    \label{fig:beta_grid}
\end{figure}

It is important to emphasize that the larger the $\beta$ parameter, the stronger the regularization of knowledge learned on the previous tasks. This implies that we put more stress on our model to remember previous tasks and limit its ability to learn a new one. Based on the results averaged over two runs, the best parameters are $\beta = 0.1$ for the CIL setup and $\beta=0.05$ for the TIL setup. However, these parameters differ from our final choice in Appendix \ref{sec:train-details}, since the parameter value $\beta=0.01$ gives better accuracy scores for both setups, based on the averaged results of 5 different runs.

% in the main section of this work, after conducting a grid search based on the results of 5 different runs, we choose the parameter value $\beta=0.01$ for both the CIL and TIL setup, since it gives better accuracy on average.

\subsection{Different interval nesting methods}

In this experimental subsection, we focus on the comparison of two different approaches to interval nesting of the input embedding to the hypernetwork. We experiment on Permuted MNIST–10 and consider two interval mappings: $\tanh(\cdot)$  and $\cos(\cdot)$. The results are shown in Figure \ref{fig:tanh_cos}. We may conclude that intervals obtained with $\tanh(\cdot)$ and $\cos(\cdot)$ behave similarly, especially on the final tasks.

\begin{figure}[htbp]
    \centering
        \includegraphics[width=0.47\textwidth]{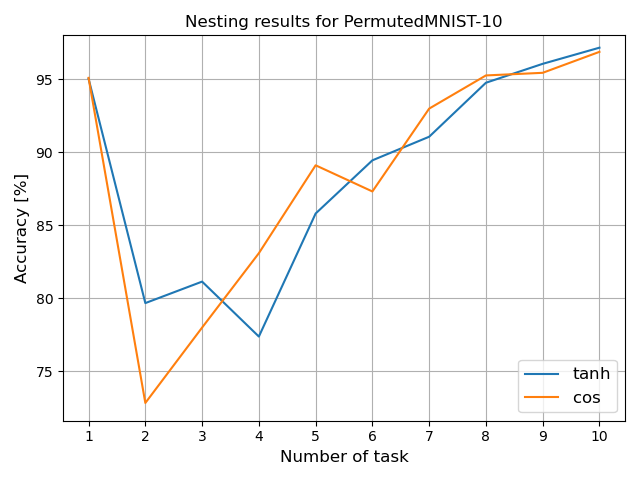}
    \caption{Mean test accuracy for consecutive continual learning (CL) tasks averaged over 2 runs using $\tanh\left(\cdot\right)$ and $\cos\left(\cdot\right)$ nesting methods for 10 tasks on the Permuted MNIST-10 dataset.
    }
    \label{fig:tanh_cos}
\end{figure}

Both $\cos(\cdot)$  and $\tanh(\cdot)$  mappings have this in common, that they scale every coordinate of the task embedding interval to be in range $[-1, 1]$. These functions also guarantee we find a non-empty intersection of task embeddings, when we multiply each coordinate by a factor of $\frac{\gamma}{M}$, where $\gamma$ is the perturbation value and $M$ is the embedding dimension.
Hence there exists a universal embedding, in the worst case being a trivial point in the embedding space.

\section{Additional experimental results} 
\label{sec:additional-exp-results}

\paragraph{TIL} Results for accuracy behaviour before and after training on consecutive tasks on TinyImageNet are shown in Figure \ref{fig:tiny_image_net_bar_plot}. Generally, the accuracy after training on all tasks is comparable to the accuracy obtained just after training on the specific task. This shows lack of catastrophic forgetting when dealing with a bigger number of tasks. Figure \ref{fig:accuracy_matrices_known_task_id} shows detailed test accuracy scores after training on consecutive tasks. We can observe small decrease in accuracy on previous tasks and overall consistency in predictions. In Figure \ref{fig:known_task_id_conf_intervals}, the $95\%$ confidence intervals are shown with the mean test accuracy for each task, averaged over 5 different runs. In case of Permuted MNIST, the average test accuracy before and after training on consecutive tasks behaves similarly. Moreover, the confidence intervals overlap each other. 

\paragraph{DIL} 
Confidence interval plots for unknown task identity are shown in Figure \ref{fig:unknown_task_id_conf_intervals}. Firstly, we observe that on Split CIFAR-100, our model is able to learn the first task but has difficulties with learning consecutive tasks. It might be because Split CIFAR-100 is too difficult to solve with one universal embedding. Secondly, large standard deviations on Split MNIST show that the \our{} training on this dataset is unstable, especially visible in accuracy on the initial tasks after learning on all tasks. Specific test accuracy scores for Permuted MIST-10 and Split MNIST, on each consecutive task, are shown in Figure \ref{fig:accuracy_matrices_unknown_task_id}. We observe positive backward transfer for the second task of Split MNIST and high accuracy values on the diagonal for Permuted MNIST-10, corresponding to the accuracy just after training on the specific task. On the right-hand side of Figure \ref{fig:known_task_id_conf_intervals} we show that \our{} can learn new tasks with high accuracy just after training for Permuted MNIST-10.

\begin{figure*}[htbp]
    \centering
    \includegraphics[width=0.32\textwidth]{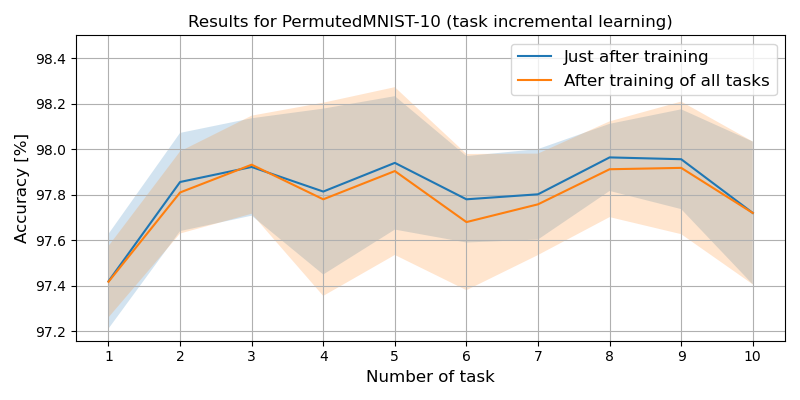}
    \includegraphics[width=0.32\textwidth]{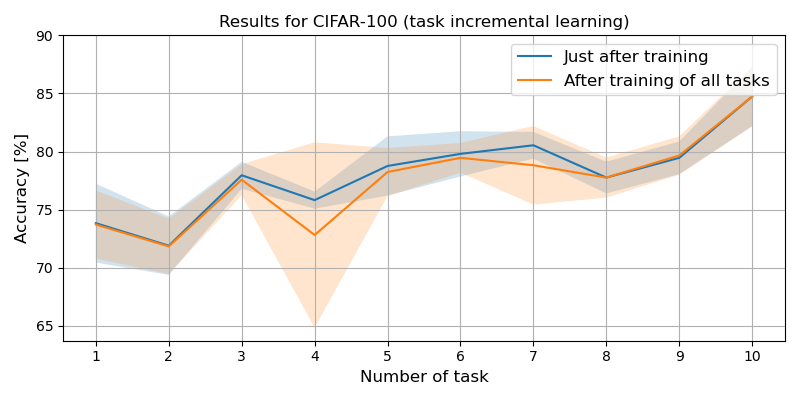}
    % {MainExperiments/acc_SplitMNIST_non_forced.png}
    \includegraphics[width=0.32\textwidth]{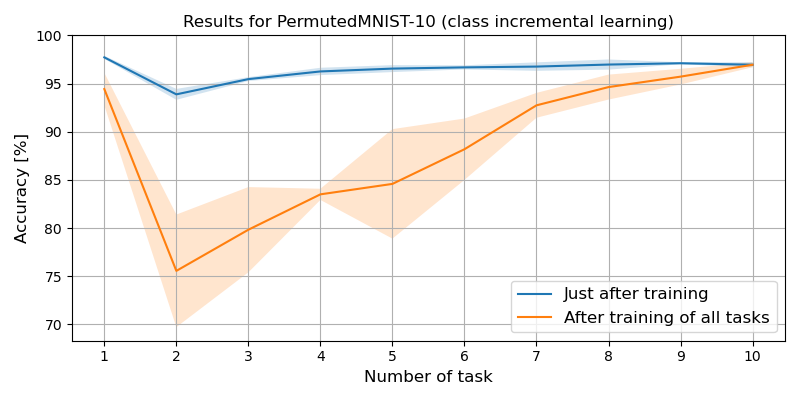}
    \caption{Average test accuracy (with 95\% confidence intervals) for Permuted MNIST-10 for 10 tasks and Split CIFAR-100 for 10 tasks.}
\label{fig:known_task_id_conf_intervals}
\end{figure*}

\begin{figure}[htbp]
    \centering
    \begin{subfigure}[b]{0.5\textwidth}
        \centering
        \hspace{-2em}\includegraphics[width=\textwidth]{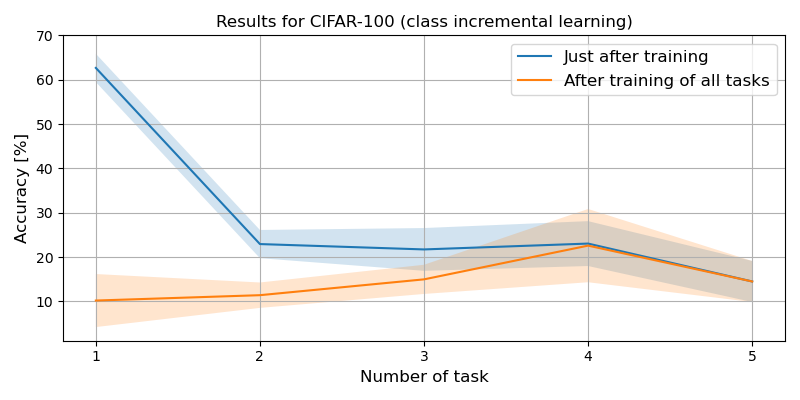}
        \caption{Results obtained using the nesting by the \newline $\cos\left(\cdot\right)$ method.}
    \end{subfigure}
    \hspace{-0.17cm}
    \begin{subfigure}[b]{0.5\textwidth}
        \centering
        \hspace{-2em}\includegraphics[width=\textwidth]{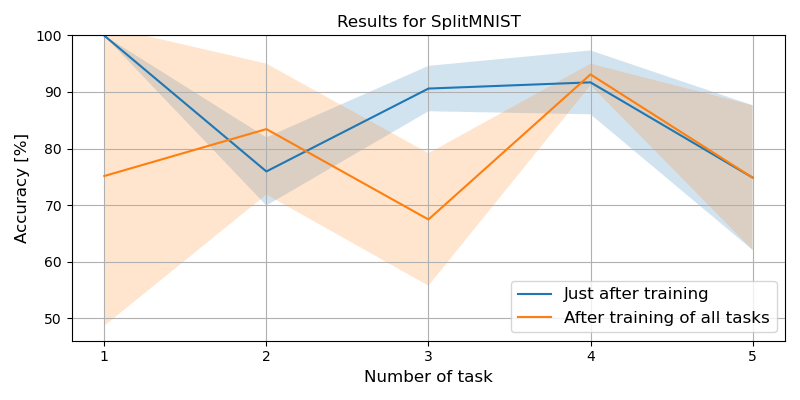}
        \caption{Results obtained using the nesting by the \newline $\cos\left(\cdot\right)$ method.}
    \end{subfigure}
    \caption{Average test accuracy (with 95\% confidence intervals) for Split CIFAR-100 and Split MNIST for 5 tasks. Results are obtained using the $\cos(\cdot)$ nesting method.}
    \label{fig:unknown_task_id_conf_intervals}
\end{figure}

\begin{figure*}[htbp]
     \centering
     \includegraphics[width=0.8\textwidth]{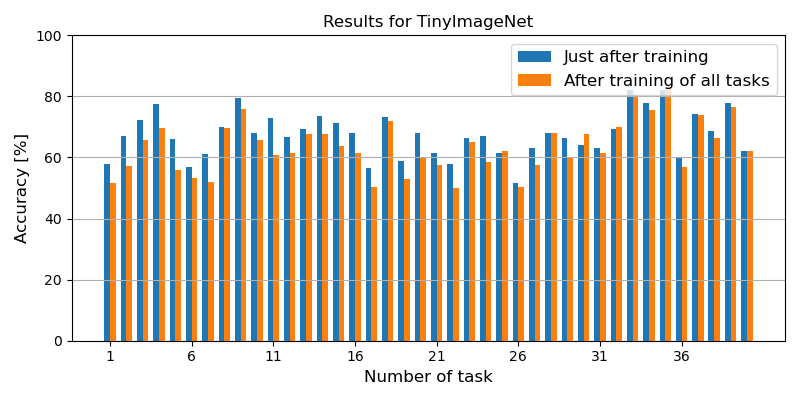}
     \caption{Test accuracy of \our{} for each task of TinyImageNet. We select one model trained on all 40 tasks, in the known task identity setup.}
 \label{fig:tiny_image_net_bar_plot}
\end{figure*}

\begin{figure}[htbp]
 \centering
     \includegraphics[width=0.5\textwidth]{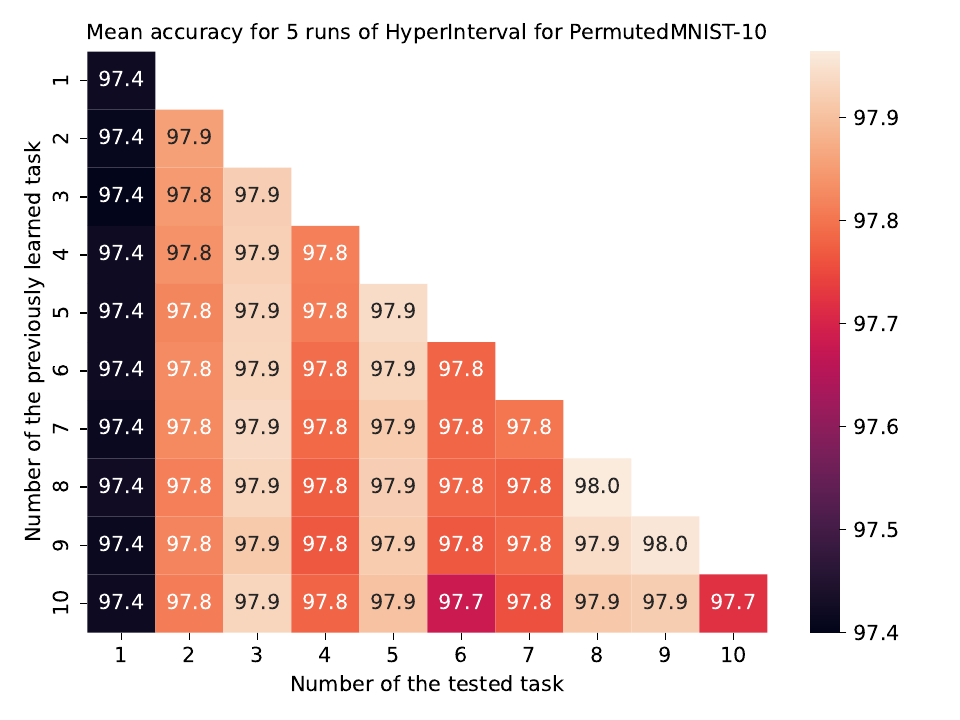}
     \hspace{-0.2cm}
     \includegraphics[width=0.5\textwidth]{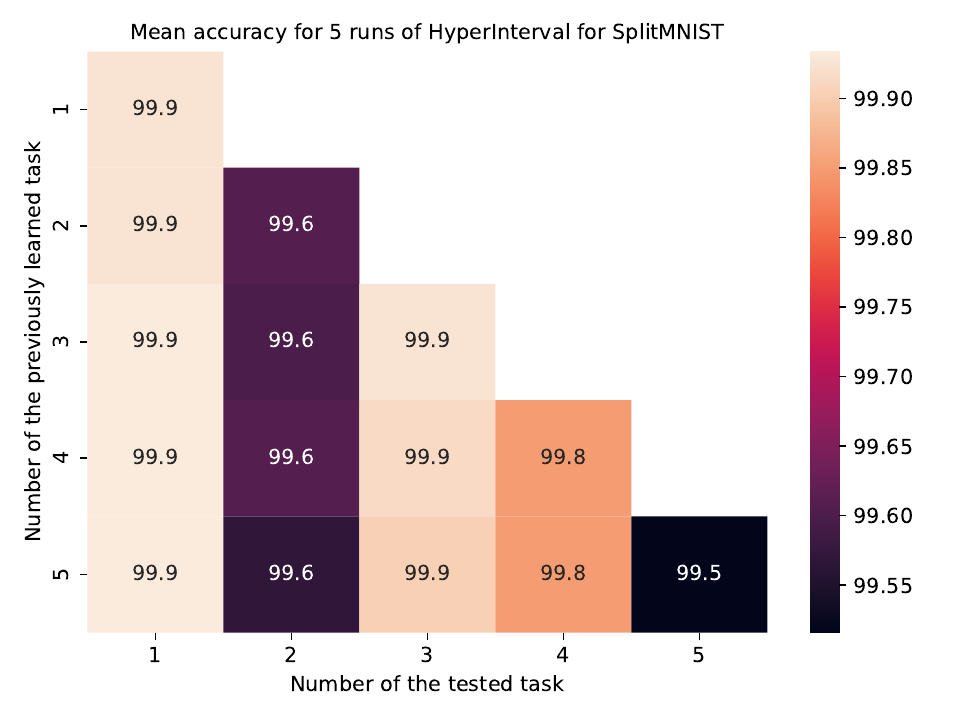}
     \includegraphics[width=0.5\textwidth]{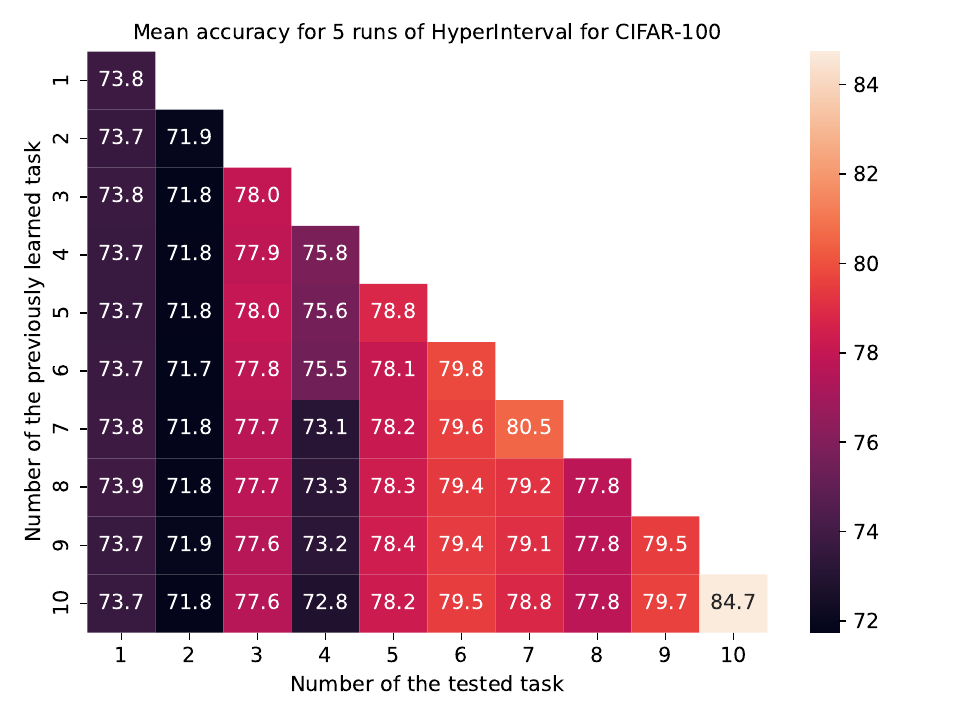}
     \caption{Mean test accuracy for consecutive CL tasks averaged over five runs of the best models. Results are obtained for the known task identity setup. The diagonal corresponds to testing the task just after the model was trained on it. The sub-diagonal values correspond to testing the task after the model was trained on consecutive tasks.}
     \label{fig:accuracy_matrices_known_task_id}
\end{figure}

\begin{figure}[htbp]
     \centering
     \begin{subfigure}[b]{0.5\textwidth}
         \centering
         % \caption{Permuted MNIST-10}
         \includegraphics[width=\textwidth]{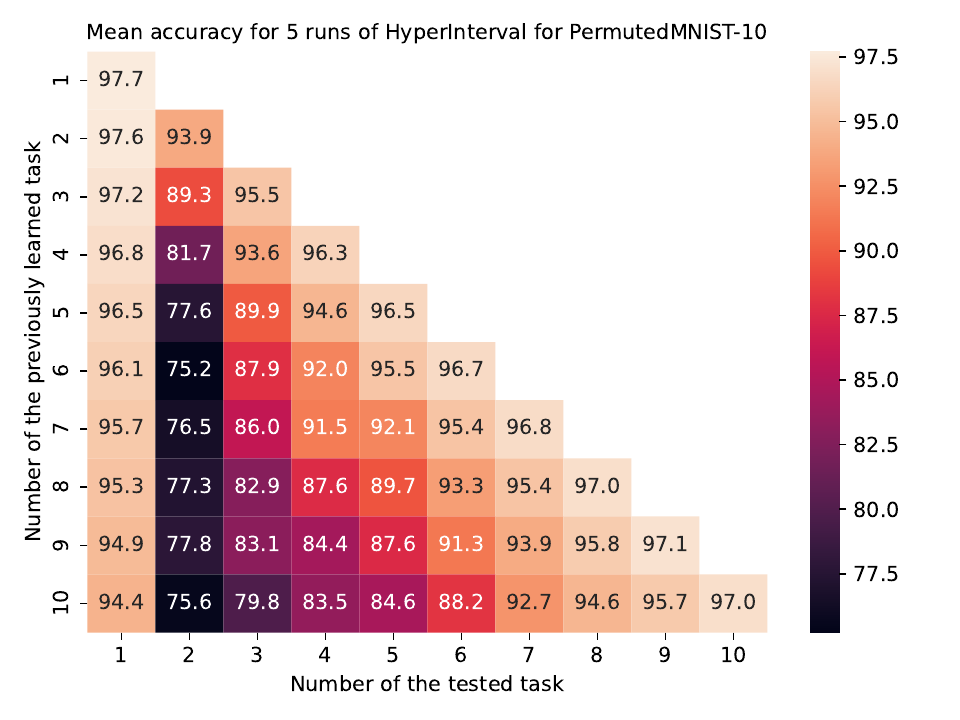}
     \end{subfigure}
     % \hfill
     \hspace{-0.5cm}
     \begin{subfigure}[b]{0.5\textwidth}
         \centering
         % \caption{Split MNIST}
         \includegraphics[width=\textwidth]{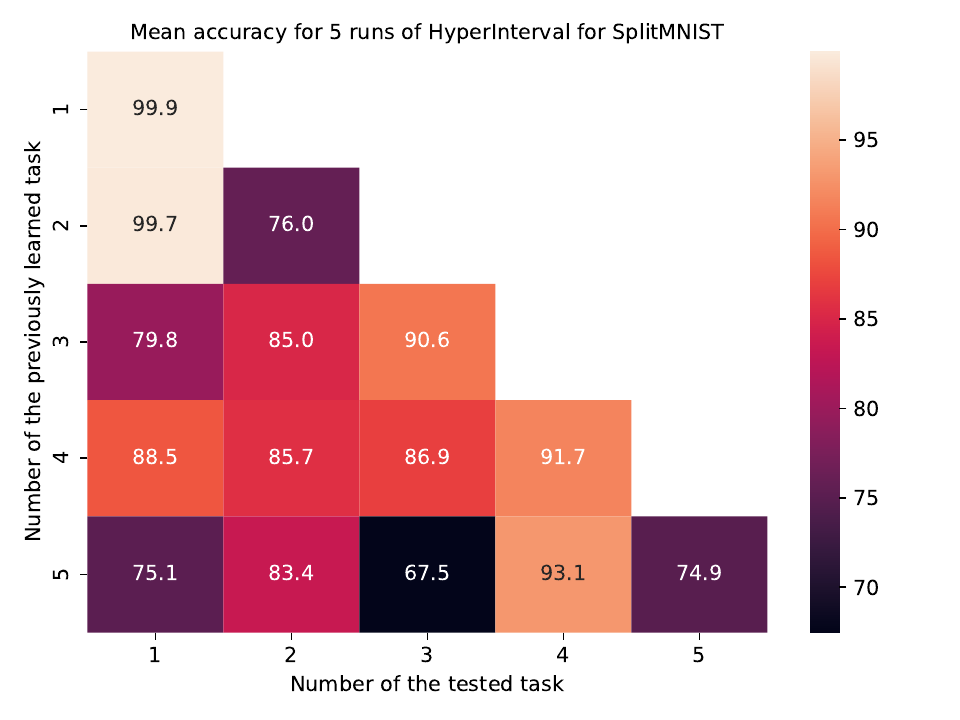}
     \end{subfigure}
     % \hfill
    
     \caption{Mean test accuracy for consecutive CL tasks averaged over five runs of the best models. Results are obtained with the $\cos(\cdot)$ nesting method.}
     \label{fig:accuracy_matrices_unknown_task_id}
\end{figure}

\subsection{Time complexity}
In this subsection, we present the training times for \our{} using the following datasets and graphics cards: 1) PermutedMNIST (DGX); 2) SplitMNIST (RTX 4090), TinyImageNet (RTX 4090), and CIFAR-100 (RTX 4090). The detailed results are shown in Table \ref{tab:time_complexity}. We report the training times for two scenarios: TIL and CIL. It is crucial to distinguish between these two scenarios because their training procedures differ.

PermutedMNIST and SplitMNIST are datasets on which \our{} was trained in both TIL and DIL scenarios simultaneously. The results indicate that \our{}'s training duration is approximately 35 minutes longer for PermutedMNIST in the DIL scenario compared to the TIL scenario. In the case of the SplitMNIST dataset, the training times are nearly identical for both scenarios. Notably, the longest training duration for \our{} is observed with the CIFAR-100 dataset (10 tasks, 10 classes per task). This extended duration can be attributed to the challenges involved in training interval convolutional layers to achieve satisfactory performance. These layers typically require more training steps to ensure a gradual increase in interval length. Additionally, a small batch size was used. Training time for the CIL scenario is the same as in the TIL scenario.

\begin{table*}[!h]
\caption{Mean time complexity averaged over 5 seeds with standard deviations for \our{} in the format HH:MM:SS.}
%\scriptsize
\centering
\begin{tabular}{@{}l@{\hskip 0.5in}c@{\hskip 0.5in}c@{}}
    \toprule
    \textbf{Dataset Name} & \textbf{TIL} & \textbf{DIL} \\
    \midrule
    PermutedMNIST & 01:11:17 $|$ 00:00:27 & 01:47:56 $|$ 00:00:21  \\
    SplitMNIST & 00:07:29 $|$ 00:00:11 & 00:07:39 $|$ 00:00:13 \\
    TinyImageNet & 00:57:37 $|$ 00:01:32 & - \\
    CIFAR-100 (10 tasks, 10 classes each) & 06:24:50 $|$ 00:07:18 & - \\
    CIFAR-100 (5 tasks, 20 classes each) & - & 02:47:26 $|$ 00:10:48 \\
    \bottomrule
\end{tabular}
\label{tab:time_complexity}
\end{table*}

\bibliography{aaai25}

\begin{thebibliography}{47}
\providecommand{\natexlab}[1]{#1}

\bibitem[{Aljundi et~al.(2018)Aljundi, Babiloni, Elhoseiny, Rohrbach, and Tuytelaars}]{aljundi2018memory}
Aljundi, R.; Babiloni, F.; Elhoseiny, M.; Rohrbach, M.; and Tuytelaars, T. 2018.
\newblock Memory Aware Synapses: Learning what (not) to forget.
\newblock In \emph{Proceedings of the European Conference on Computer Vision (ECCV)}.

\bibitem[{Aljundi et~al.(2019)Aljundi, Belilovsky, Tuytelaars, Charlin, Caccia, Lin, and Page-Caccia}]{2019aljundi+6}
Aljundi, R.; Belilovsky, E.; Tuytelaars, T.; Charlin, L.; Caccia, M.; Lin, M.; and Page-Caccia, L. 2019.
\newblock Online continual learning with maximal interfered retrieval.
\newblock In \emph{Advances in Neural Information Processing Systems (NeurIPS)}, volume~32. Curran Associates, Inc.

\bibitem[{Belouadah and Popescu(2019)}]{2019belouadah+1}
Belouadah, E.; and Popescu, A. 2019.
\newblock {IL2M}: Class incremental learning with dual memory.
\newblock In \emph{International Conference on Computer Vision (ICCV)}.

\bibitem[{Castro et~al.(2018)Castro, Marín-Jiménez, Guil, Schmid, and Alahari}]{2018castro+4}
Castro, F.~M.; Marín-Jiménez, M.~J.; Guil, N.; Schmid, C.; and Alahari, K. 2018.
\newblock End-to-end incremental learning.
\newblock In \emph{European Conference on Computer Vision (ECCV)}.

\bibitem[{Chaudhry et~al.(2018)Chaudhry, Dokania, Ajanthan, and Torr}]{2018chaundry+3}
Chaudhry, A.; Dokania, P.~K.; Ajanthan, T.; and Torr, P.~H. 2018.
\newblock Riemannian walk for incremental learning: Understanding forgetting and intransigence.
\newblock In \emph{Proceedings of the European conference on computer vision (ECCV)}, 532--547.

\bibitem[{Cywiński et~al.(2024)Cywiński, Deja, Trzciński, Twardowski, and Łukasz Kuciński}]{2024cywinski+4}
Cywiński, B.; Deja, K.; Trzciński, T.; Twardowski, B.; and Łukasz Kuciński. 2024.
\newblock GUIDE: Guidance-based Incremental Learning with Diffusion Models.
\newblock ArXiv:2403.03938.

\bibitem[{Dahlquist and Bj{\"o}rck(2008)}]{dahlquist2008numerical}
Dahlquist, G.; and Bj{\"o}rck, {\AA}. 2008.
\newblock \emph{Numerical methods in scientific computing, volume I}.
\newblock SIAM.

\bibitem[{Deng and Russakovsky(2022)}]{2022deng+1}
Deng, Z.; and Russakovsky, O. 2022.
\newblock Remember the past: Distilling datasets into addressable memories for neural networks.
\newblock In \emph{Advances in Neural Information Processing Systems (NeurIPS)}, volume~35, 34391--34404.

\bibitem[{French(1999)}]{french1999catastrophic}
French, R.~M. 1999.
\newblock Catastrophic forgetting in connectionist networks.
\newblock \emph{Trends in cognitive sciences}, 3(4): 128--135.

\bibitem[{Goswami et~al.(2024)Goswami, Liu, Twardowski, and van~de Weijer}]{goswami2024fecam}
Goswami, D.; Liu, Y.; Twardowski, B.; and van~de Weijer, J. 2024.
\newblock FeCAM: Exploiting the Heterogeneity of Class Distributions in Exemplar-Free Continual Learning.
\newblock arXiv:2309.14062.

\bibitem[{Gowal et~al.(2018)Gowal, Dvijotham, Stanforth, Bunel, Qin, Uesato, Arandjelovic, Mann, and Kohli}]{gowal2018effectiveness}
Gowal, S.; Dvijotham, K.; Stanforth, R.; Bunel, R.; Qin, C.; Uesato, J.; Arandjelovic, R.; Mann, T.; and Kohli, P. 2018.
\newblock On the effectiveness of interval bound propagation for training verifiably robust models.
\newblock \emph{arXiv preprint arXiv:1810.12715}.

\bibitem[{Ha, Dai, and Le(2016)}]{ha2016hypernetworks}
Ha, D.; Dai, A.; and Le, Q.~V. 2016.
\newblock Hypernetworks.
\newblock \emph{arXiv preprint arXiv:1609.09106}.

\bibitem[{Henning et~al.(2021)Henning, Cervera, D'Angelo, Von~Oswald, Traber, Ehret, Kobayashi, Grewe, and Sacramento}]{henning2021posterior}
Henning, C.; Cervera, M.; D'Angelo, F.; Von~Oswald, J.; Traber, R.; Ehret, B.; Kobayashi, S.; Grewe, B.~F.; and Sacramento, J. 2021.
\newblock Posterior meta-replay for continual learning.
\newblock \emph{Advances in Neural Information Processing Systems}, 34: 14135--14149.

\bibitem[{Hou et~al.(2019)Hou, Pan, Loy, Wang, and Lin}]{2019hou+5}
Hou, S.; Pan, X.; Loy, C.~C.; Wang, Z.; and Lin, D. 2019.
\newblock Learning a unified classifier incrementally via rebalancing.
\newblock In \emph{International Conference on Computer Vision (ICCV)}.

\bibitem[{Kang et~al.(2023)Kang, Yoon, Madjid, Hwang, and Yoo}]{kang2023forgetfree}
Kang, H.; Yoon, J.; Madjid, S.~R.; Hwang, S.~J.; and Yoo, C.~D. 2023.
\newblock Forget-free Continual Learning with Soft-Winning SubNetworks.
\newblock arXiv:2303.14962.

\bibitem[{Kirkpatrick et~al.(2017)Kirkpatrick, Pascanu, Rabinowitz, Veness, Desjardins, Rusu, Milan, Quan, Ramalho, Grabska-Barwinska et~al.}]{kirkpatrick2017overcoming}
Kirkpatrick, J.; Pascanu, R.; Rabinowitz, N.; Veness, J.; Desjardins, G.; Rusu, A.~A.; Milan, K.; Quan, J.; Ramalho, T.; Grabska-Barwinska, A.; et~al. 2017.
\newblock Overcoming catastrophic forgetting in neural networks.
\newblock \emph{Proceedings of the National Academy of Sciences}, 114(13): 3521--3526.

\bibitem[{Ksi{\k{a}}{\.z}ek and Spurek(2023)}]{ksikazek2023hypermask}
Ksi{\k{a}}{\.z}ek, K.; and Spurek, P. 2023.
\newblock HyperMask: Adaptive Hypernetwork-based Masks for Continual Learning.
\newblock \emph{arXiv preprint arXiv:2310.00113}.

\bibitem[{Lee(2004)}]{lee2004first}
Lee, K.~H. 2004.
\newblock \emph{First course on fuzzy theory and applications}, volume~27.
\newblock Springer Science \& Business Media.

\bibitem[{Li and Hoiem(2017)}]{li2017learning}
Li, Z.; and Hoiem, D. 2017.
\newblock Learning without forgetting.
\newblock \emph{IEEE transactions on pattern analysis and machine intelligence}, 40(12): 2935--2947.

\bibitem[{Liu et~al.(2020)Liu, Liu, Su, Schiele, and Sun}]{2020liu+4}
Liu, Y.; Liu, A.-A.; Su, Y.; Schiele, B.; and Sun, Q. 2020.
\newblock Mnemonics training: Multi-class incremental learning without forgetting.
\newblock In \emph{Computer Vision and Pattern Recognition (CVPR)}.

\bibitem[{Lopez-Paz and Ranzato(2017)}]{lopezpaz2017gradientEM}
Lopez-Paz, D.; and Ranzato, M. 2017.
\newblock Gradient Episodic Memory for Continual Learning.
\newblock In \emph{Proceedings of the Advances in Neural Information Processing Systems (NeurIPS)}.

\bibitem[{Mallya, Davis, and Lazebnik(2018)}]{Mallya2018PiggybackAA}
Mallya, A.; Davis, D.; and Lazebnik, S. 2018.
\newblock Piggyback: Adapting a Single Network to Multiple Tasks by Learning to Mask Weights.
\newblock In \emph{Proceedings of the European Conference on Computer Vision (ECCV)}.

\bibitem[{Mallya and Lazebnik(2018)}]{mallya2018packnet}
Mallya, A.; and Lazebnik, S. 2018.
\newblock Packnet: Adding multiple tasks to a single network by iterative pruning.
\newblock In \emph{Proceedings of the IEEE Conference on Computer Vision and Pattern Recognition (CVPR)}.

\bibitem[{Masse, Grant, and Freedman(2018)}]{Masse2018AlleviatingCF}
Masse, N.~Y.; Grant, G.~D.; and Freedman, D.~J. 2018.
\newblock Alleviating catastrophic forgetting using context-dependent gating and synaptic stabilization.
\newblock \emph{Proceedings of the National Academy of Sciences}, 115: E10467 -- E10475.

\bibitem[{McCloskey and Cohen(1989)}]{mccloskey1989catastrophic}
McCloskey, M.; and Cohen, N.~J. 1989.
\newblock Catastrophic interference in connectionist networks: The sequential learning problem.
\newblock In \emph{Psychology of learning and motivation}, volume~24, 109--165. Elsevier.

\bibitem[{Nguyen et~al.(2018)Nguyen, Li, Bui, and Turner}]{2018nguyen+3}
Nguyen, C.~V.; Li, Y.; Bui, T.~D.; and Turner, R.~E. 2018.
\newblock Variational continual learning.
\newblock In \emph{International Conference on Learning Representations (ICLR)}.

\bibitem[{Prabhu, Torr, and Dokania(2020)}]{2020prabhu+2}
Prabhu, A.; Torr, P.~H.; and Dokania, P.~K. 2020.
\newblock Gdumb: A simple approach that questions our progress in continual learning.
\newblock In \emph{Proceedings of the European Conference on Computer Vision (ECCV)}, 524--540.

\bibitem[{Ratcliff(1990)}]{ratcliff1990connectionist}
Ratcliff, R. 1990.
\newblock Connectionist models of recognition memory: constraints imposed by learning and forgetting functions.
\newblock \emph{Psychological review}, 97(2): 285.

\bibitem[{Rebuffi et~al.(2017)Rebuffi, Kolesnikov, Sperl, and Lampert}]{2017rebuffi+3}
Rebuffi, S.; Kolesnikov, A.; Sperl, G.; and Lampert, C.~H. 2017.
\newblock iCaRL: Incremental Classifier and Representation Learning.
\newblock In \emph{Computer Visiona and Pattern Recognition (CVPR)}.

\bibitem[{Rolnick et~al.(2019)Rolnick, Ahuja, Schwarz, Lillicrap, and Wayne}]{Rolnick2019ExperienceRF}
Rolnick, D.; Ahuja, A.; Schwarz, J.; Lillicrap, T.~P.; and Wayne, G. 2019.
\newblock Experience Replay for Continual Learning.
\newblock In \emph{Proceedings of the Advances in Neural Information Processing Systems (NeurIPS)}.

\bibitem[{Rostami, Kolouri, and Pilly(2019)}]{2019rostami+2}
Rostami, M.; Kolouri, S.; and Pilly, P.~K. 2019.
\newblock Complementary learning for overcoming catastrophic forgetting using experience replay.
\newblock In \emph{International Joint Conference on Artificial Intelligence (IJCAI)}, 3339--3345.

\bibitem[{Russakovsky et~al.(2015)Russakovsky, Deng, Su, Krause, Satheesh, Ma, Huang, Karpathy, Khosla, Bernstein, Berg, and Fei-Fei}]{russakovsky2015imagenet}
Russakovsky, O.; Deng, J.; Su, H.; Krause, J.; Satheesh, S.; Ma, S.; Huang, Z.; Karpathy, A.; Khosla, A.; Bernstein, M.; Berg, A.~C.; and Fei-Fei, L. 2015.
\newblock ImageNet Large Scale Visual Recognition Challenge.
\newblock arXiv:1409.0575.

\bibitem[{Rusu et~al.(2016)Rusu, Rabinowitz, Desjardins, Soyer, Kirkpatrick, Kavukcuoglu, Pascanu, and Hadsell}]{rusu2016progressive}
Rusu, A.~A.; Rabinowitz, N.~C.; Desjardins, G.; Soyer, H.; Kirkpatrick, J.; Kavukcuoglu, K.; Pascanu, R.; and Hadsell, R. 2016.
\newblock Progressive neural networks.
\newblock \emph{arXiv preprint arXiv:1606.04671}.

\bibitem[{Shin et~al.(2017)Shin, Lee, Kim, and Kim}]{Shin2017ContinualLW}
Shin, H.; Lee, J.~K.; Kim, J.; and Kim, J. 2017.
\newblock Continual Learning with Deep Generative Replay.
\newblock In \emph{Proceedings of the Advances in Neural Information Processing Systems (NeurIPS)}.

\bibitem[{van~de Ven, Siegelmann, and Tolias(2020)}]{vandeVen2020BraininspiredRF}
van~de Ven, G.~M.; Siegelmann, H.~T.; and Tolias, A.~S. 2020.
\newblock Brain-inspired replay for continual learning with artificial neural networks.
\newblock \emph{Nature Communications}, 11.

\bibitem[{{van de Ven} and Tolias(2018)}]{2018vandeven+1}
{van de Ven}, G.~M.; and Tolias, A.~S. 2018.
\newblock Generative replay with feedback connections as a general strategy for continual learning.
\newblock ArXiv:1809.10635.

\bibitem[{van~de Ven and Tolias(2019)}]{2019vandeven+1}
van~de Ven, G.~M.; and Tolias, A.~S. 2019.
\newblock Three scenarios for continual learning.
\newblock ArXiv:1904.07734.

\bibitem[{Virmaux and Scaman(2018)}]{scaman2019lipschitzregularitydeepneural}
Virmaux, A.; and Scaman, K. 2018.
\newblock Lipschitz regularity of deep neural networks: analysis and efficient estimation.
\newblock In Bengio, S.; Wallach, H.; Larochelle, H.; Grauman, K.; Cesa-Bianchi, N.; and Garnett, R., eds., \emph{Advances in Neural Information Processing Systems}, volume~31. Curran Associates, Inc.

\bibitem[{von Oswald et~al.(2019)von Oswald, Henning, Grewe, and Sacramento}]{von2019continual}
von Oswald, J.; Henning, C.; Grewe, B.~F.; and Sacramento, J. 2019.
\newblock Continual learning with hypernetworks.
\newblock In \emph{International Conference on Learning Representations}.

\bibitem[{von Oswald et~al.(2020)von Oswald, Henning, Grewe, and Sacramento}]{oshg2019hypercl}
von Oswald, J.; Henning, C.; Grewe, B.~F.; and Sacramento, J. 2020.
\newblock Continual learning with hypernetworks.
\newblock In \emph{International Conference on Learning Representations}.

\bibitem[{Wang et~al.(2018)Wang, Zhu, Torralba, and Efros}]{2018wang+3}
Wang, T.; Zhu, J.-Y.; Torralba, A.; and Efros, A.~A. 2018.
\newblock Dataset distillation.
\newblock ArXiv:1811.10959.

\bibitem[{Wo{\l}czyk et~al.(2022)Wo{\l}czyk, Piczak, W{\'o}jcik, Pustelnik, Morawiecki, Tabor, Trzcinski, and Spurek}]{wolczyk2022continual}
Wo{\l}czyk, M.; Piczak, K.; W{\'o}jcik, B.; Pustelnik, L.; Morawiecki, P.; Tabor, J.; Trzcinski, T.; and Spurek, P. 2022.
\newblock Continual learning with guarantees via weight interval constraints.
\newblock In \emph{International Conference on Machine Learning}, 23897--23911. PMLR.

\bibitem[{Wortsman et~al.(2020)Wortsman, Ramanujan, Liu, Kembhavi, Rastegari, Yosinski, and Farhadi}]{2020wortsman+6}
Wortsman, M.; Ramanujan, V.; Liu, R.; Kembhavi, A.; Rastegari, M.; Yosinski, J.; and Farhadi, A. 2020.
\newblock Supermasks in superposition.
\newblock In \emph{Neural Information Processing Systems (NeurIPS)}, 15173--15184.

\bibitem[{Wu et~al.(2019)Wu, Chen, Wang, Ye, Liu, Guo, and Fu}]{2019wu+5}
Wu, Y.; Chen, Y.; Wang, L.; Ye, Y.; Liu, Z.; Guo, Y.; and Fu, Y. 2019.
\newblock Large scale incremental learning.
\newblock In \emph{International Conference on Computer Vision (ICCV)}.

\bibitem[{Yoon et~al.(2018)Yoon, Yang, Lee, and Hwang}]{2018yoon+3}
Yoon, J.; Yang, E.; Lee, J.; and Hwang, S.~J. 2018.
\newblock Lifelong Learning with Dynamically Expandable Networks.
\newblock In \emph{International Conference on Learning Representations (ICLR)}.

\bibitem[{Zenke, Poole, and Ganguli(2017)}]{2017zenke+2}
Zenke, F.; Poole, B.; and Ganguli, S. 2017.
\newblock Continual learning through synaptic intelligence.
\newblock In \emph{International Conference on Machine Learning (ICML)}, 3987--3995.

\bibitem[{Zhao, Mopuri, and Bilen(2021)}]{2021zhao+2}
Zhao, B.; Mopuri, K.~R.; and Bilen, H. 2021.
\newblock Dataset condensation with gradient matching.
\newblock In \emph{International Conference on Learning Representations (ICLR)}.

\end{thebibliography}

\end{document}